\documentclass{article}

\PassOptionsToPackage{authoryear, round}{natbib}

\usepackage[final]{neurips_2022}

\usepackage[utf8]{inputenc} %
\usepackage[T1]{fontenc}    %
\usepackage{hyperref}       %
\usepackage{url}            %
\usepackage{booktabs}       %
\usepackage{amsfonts}       %
\usepackage{nicefrac}       %
\usepackage{microtype}      %
\usepackage[usenames,dvipsnames]{xcolor}         %

\usepackage{amsmath}
\usepackage{amssymb}
\usepackage{amsthm}
\usepackage{mathtools}
\usepackage{dsfont}
\usepackage{mathrsfs}

\usepackage{cleveref}
\usepackage{thm-restate}  %

\usepackage{apptools}
\AtAppendix{\counterwithin{lemma}{section}}

\newcommand{\StateSpace}{\mathcal{S}}
\newcommand{\ActionSpace}{\mathcal{A}}
\newcommand{\StateSpaceSize}{S}
\newcommand{\ActionSpaceSize}{A}
\newcommand{\TransitionModel}{P}

\newcommand{\reward}{r}
\newcommand{\Horizon}{H}
\newcommand{\MDP}{\mathcal{M}}

\newcommand{\EstMDP}{\widehat{\MDP}}
\newcommand{\EstTransitionModel}{\widehat{\TransitionModel}}

\newcommand{\MDPdef}{\MDP \coloneqq (\StateSpace, \ActionSpace, \TransitionModel, \Horizon, \state_0)}
\newcommand{\MDPnoR}{MDP\ensuremath{\setminus}R\xspace}

\newcommand{\state}{s}
\newcommand{\action}{a}
\newcommand{\timestep}{h}

\newcommand{\Policy}{\pi}
\newcommand{\ExpertPolicy}{\Policy^{E}}

\newcommand{\GoodEvent}{\mathscr{E}}

\newcommand{\identity}{I}

\newcommand{\Qfun}[3]{Q_{#1}^{#2,#3}}
\newcommand{\Valfun}[3]{V_{#1}^{#2, #3}}
\newcommand{\EstValfun}[3]{\hat{V}_{#1}^{#2, #3}}

\newcommand{\TrueRew}{\reward}

\newcommand{\EstPol}{\hat{\Policy}}
\newcommand{\reals}{\mathbb{R}}

\newcommand{\Occupancy}[4]{\eta_{#1,#2}^{#3,#4}}

\newcommand{\indicator}[1]{\mathds{1}_{\{#1\}}}

\newcommand{\Rmax}{R_{\mathrm{max}}}
\renewcommand{\Pr}{\mathrm{Pr}}

\newcommand{\rewardci}{C}
\newcommand{\policyci}{\hat{\Pi}}

\newcommand{\irlalg}{\mathscr{A}}

\newcommand{\argmin}{\mathrm{argmin}}

\newcommand{\ExactFeasibleSet}{\mathcal{R}_{\mathfrak{B}}}
\newcommand{\RecoveredFeasibleSet}{\mathcal{R}_{\hat{\mathfrak{B}}}}
\newcommand{\RecoveredFeasibleSetIter}{\mathcal{R}_{\hat{\mathfrak{B}}_\episode}}
\newcommand{\RecoveredFeasibleSetIterPrime}{\mathcal{R}_{\hat{\mathfrak{B}}_{\episode'}}}

\newcommand{\episode}{k}

\newtheorem{theorem}{Theorem}
\newtheorem{definition}[theorem]{Definition}

\newtheorem{lemma}[theorem]{Lemma}
\newtheorem{corollary}[theorem]{Corollary}

\usepackage{algorithm}
\usepackage{algpseudocode}

\usepackage[toc,page,header]{appendix}
\usepackage{minitoc}

\newcommand{\specialcell}[2][c]{%
  \begin{tabular}[#1]{@{}l@{}}#2\end{tabular}}

\usepackage{wrapfig}

\usepackage{subfigure}

\definecolor{mygreen}{rgb}{0.0, 0.5, 0.0}

\usepackage{xspace}

\newcommand{\AlgNameShort}{AceIRL\xspace}
\newcommand{\AlgNameDef}{\textbf{Ac}tive \textbf{e}xploration for \textbf{I}nverse \textbf{R}einforcement \textbf{L}earning (\AlgNameShort)\xspace}

\newcommand{\codelink}{\url{https://github.com/lasgroup/aceirl}}

\title{Active Exploration for \\ Inverse Reinforcement Learning}

\author{%
    David Lindner \\
    Department of Computer Science \\
    ETH Zurich \\
    \texttt{david.lindner@inf.ethz.ch}
    \And
    Andreas Krause \\
    Department of Computer Science \\
    ETH Zurich \\
    \texttt{krausea@ethz.ch}
    \AND
    Giorgia Ramponi \\
    ETH AI Center \\
    \texttt{giorgia.ramponi@ai.ethz.ch}
}

\begin{document}

\maketitle

\begin{abstract}

\looseness -1 Inverse Reinforcement Learning (IRL) is a powerful paradigm for inferring a reward function from expert demonstrations.
Many IRL algorithms require a known transition model and sometimes even a known expert policy, or they at least require access to a generative model. However, these assumptions are too strong for many real-world applications, where the environment can be accessed only through sequential interaction.
We propose a novel IRL algorithm: \AlgNameDef, which actively explores an unknown environment and expert policy to quickly learn the expert's reward function and identify a good policy. \AlgNameShort uses previous observations to construct confidence intervals that capture plausible reward functions and find exploration policies that focus on the most informative regions of the environment.
\AlgNameShort is the first approach to active IRL with sample-complexity bounds that does not require a generative model of the environment. \AlgNameShort matches the sample complexity of active IRL with a generative model in the worst case. Additionally, we establish a problem-dependent bound that relates the sample complexity of \AlgNameShort to the suboptimality gap of a given IRL problem.
We empirically evaluate \AlgNameShort in simulations and find that it significantly outperforms more naive exploration strategies.
\end{abstract}

\section{Introduction}
\label{sec:introduction}
Reinforcement Learning \citep[RL;][]{sutton2018reinforcement} has achieved impressive results, from playing video games \citep{mnih2015human} to solving robotic control problems \citep{haarnoja2018learning}. However, in many applications, it is challenging to design a reward function that robustly describes the desired task \citep{amodei2016concrete,hendrycks2021unsolved}.
Instead of using an explicit reward function, Inverse Reinforcement Learning \cite[IRL;][]{ng2000algorithms} seeks to recover the reward by observing an \textit{expert}, e.g., an human who already knows how to perform a task. However, most existing IRL algorithms assume that the transition model, and in some cases, the expert's policy, are {\em known}. In many real-world applications, this is not given, and the agent needs to estimate the transition dynamics and the expert policy from samples.
\Cref{fig:motivating_example} shows an illustrative example where the agent can choose between different paths that have different properties, e.g., walking speeds, and lead to different goals. 
The agent has to explore the environment and query the expert policy in order to infer the expert's reward function.

IRL with sample-based estimation was only recently analyzed formally by \citet{metelli2021provably}. They decompose the error on the reward into a contribution from estimating the transition model and estimating the expert
policy. Based on this, \citet{metelli2021provably} propose an efficient sampling strategy to recover a good reward function. However, they assume a \emph{generative model} of the environment, i.e., the agent can query the transition dynamics for arbitrary states and actions. In practice, this assumption is unrealistic. The agent in \Cref{fig:motivating_example} starts in the middle and cannot learn about the properties of any path without actually \emph{exploring the environment}.

In this work, we consider IRL with unknown transition dynamics and expert policy and focus on exploring the environment in order to recover the expert's reward function efficiently. To the best of our knowledge, we present the first paper providing sample complexity guarantees for the active IRL problem without access to a generative model.

Our main contributions are:
\begin{itemize}
    \item We propose the active IRL problem in a finite-horizon, undiscounted Markov Decision Process (MDP) and characterize necessary and sufficient conditions for solving it (\Cref{sec:problem-definition}).
    \item We analyze how the estimation errors of the transition model and the expert policy contribute to the estimation error of the reward function, extending prior work to the finite-horizon setting. We provide a novel analysis of how this error affects the performance of the policy, which optimizes the recovered reward function (\Cref{sec:feasible-rewards}).
    \item We propose a novel algorithm, \AlgNameDef, which actively explores the environment and the expert policy to infer a good reward function. In each iteration, \AlgNameShort constructs an exploratory policy based on the estimation error of the recovered reward function (\Cref{sec:aceirl}).
    \item \looseness -1 We consider two different exploration strategies for \AlgNameShort. The first, more straightforward, strategy provides a sample complexity similar to the algorithm proposed by \citet{metelli2021provably}, which has access to a generative model (\Cref{sec:stopping-condition}). The second strategy takes the expected reduction in uncertainty into account (\Cref{sec:stopping-condition-dep}). This yields a tighter, problem-dependent sample complexity bound at the cost of solving a convex optimization problem in each iteration (\Cref{sec:theoretical_results}).
    \item We evaluate \AlgNameShort empirically in simulated environments and demonstrate that it achieves significantly better performance than more naive exploration strategies (\Cref{sec:experiments}). We provide code to reproduce out experiments at \codelink.
\end{itemize}

The proofs of all results presented in the main paper can be found in \Cref{app:theory}.

\begin{figure}
    \centering
    \includegraphics[width=0.8\linewidth]{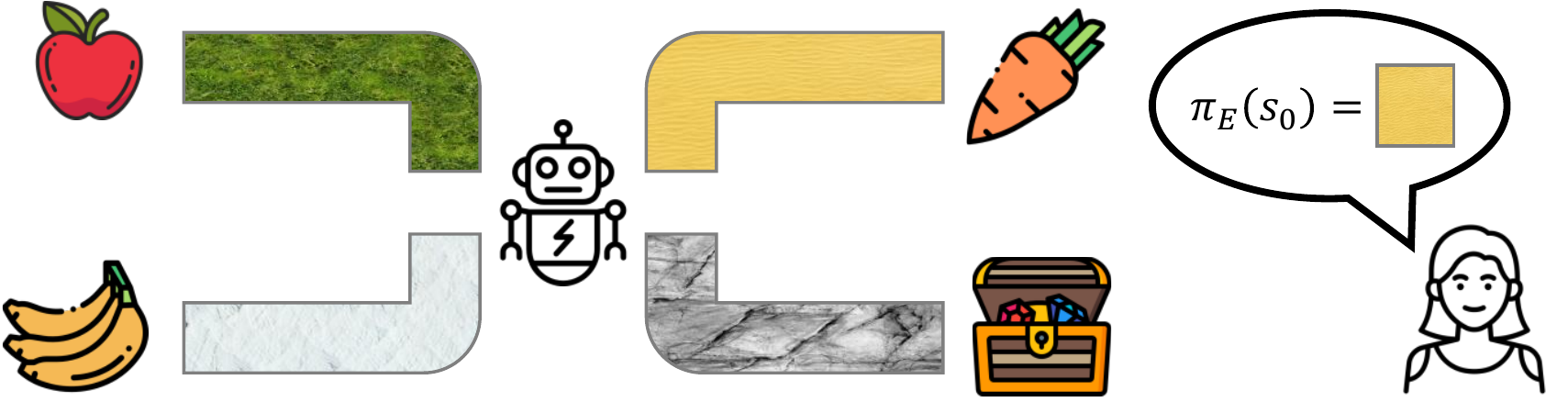}
    \caption{An illustrative example of Active IRL. The agent can choose between four paths that lead to different objects. It can get action recommendations from an expert but does not know about the properties of the different paths (the transition dynamics) or the value of different items (the reward function). The agent's goal is to infer which reward functions explain the expert's recommendations. Only observing the expert actions is not enough to do that. Rather, the agent has to explore the environment and learn about the dynamics. The human might prefer to find the treasure over the carrot but still recommend the yellow path because  the treasure is very difficult to reach. To explore efficiently, the agent has to combine its uncertainty about the expert policy with its uncertainty about the environment to choose where to explore. \AlgNameShort implements an exploration strategy that aims to infer which reward functions are consistent with the expert's recommendations as quickly as possible. We present experiments on a version of this environment in \Cref{sec:experiments}.}
    \label{fig:motivating_example}
\end{figure}

\section{Related Work}
\label{sec:rel-works}

Most IRL algorithms assume that the underlying transition model is known \citep{ratliff2006maximum,ziebart2008maximum,ramachandran2007bayesian,levine2011nonlinear}. However, the transition model usually needs to be estimated from samples, which induces an error in the recovered reward function that most papers do not study. \citet{metelli2021provably} analyze this error and the sample complexity of IRL in a tabular setting with a generative model. They propose an algorithm focused on transferring the learned reward function to a fully known target environment.
\citet{dexter2021inverse} provides a similar analysis in continuous state spaces and discrete action spaces, but they still require a generative model of the environment.
In contrast, we {\em do not} assume access to a generative model and thus need to tackle the exploration problem in IRL.

Some prior work studies active learning algorithms for IRL in a Bayesian framework but without theoretical guarantees.
\citet{lopes2009active} propose an active learning algorithm for IRL that estimates a posterior distribution over reward functions from demonstrations, requiring a prior distribution and full knowledge of the environment dynamics.  Relatedly, \citet{cohn2011comparing} consider a Bayesian IRL setting with a semi-autonomous agent that asks an expert for advice if it is uncertain about the reward.
\citet{brown2018risk} empirically study active IRL in several safety-critical environments, selecting queries using value at risk. \citet{kulick2013active} consider active learning for a robotic manipulation task, asking a human expert for advice in situations with the highest predictive uncertainty. 
Similarly, \citet{losey2018including} propose a method to learn uncertainty estimates from human corrections in a robotics context. All of these papers assume a Bayesian framework and do not provide theoretical guarantees. In contrast, our setup does not require a prior over reward functions, and we provide theoretical sample complexity guarantees for our algorithm.

A separate line of work studies sample complexity in \emph{imitation learning} where the goal is to imitate an expert policy rather than infer a reward function \citep{rajaraman2020toward,xu2020error}. In particular, \citet{abbeel2005exploration} also focus on exploration and propose to use the expert policy to explore relevant regions, whereas \citet{shani2022online} use an upper-confidence approach to exploration. Our setting is different because we focus on IRL instead of imitation learning, and we aim to explore to infer a reward function learn most effectively.

\section{Preliminaries}
\label{sec:preliminaries}

Let us first introduce necessary background and notation that we use throughout the paper.

\looseness -1 \textbf{Markov decision process.} A finite-horizon (or episodic) Markov Decision Process without reward function (\MDPnoR) is a tuple $\MDPdef$, where $\StateSpace$ is the finite state space of size $\StateSpaceSize$; $\ActionSpace$ is the  finite action space of size $\ActionSpaceSize$; $\TransitionModel: \StateSpace \times \ActionSpace \to \Delta_{\StateSpace}$ is the transition model; $\Horizon$ is the horizon and $\state_0$ is the initial state. \footnote{We can model any initial state distribution as a single initial state by modifying the transitions.} In other words, a finite-horizon \MDPnoR is a finite-horizon MDP \citep{puterman2014markov} without the reward function. We describe an agent's behaviour with a (possible stochastic) policy $\pi \in \StateSpace\times[\Horizon] \to \Delta_{\ActionSpace}$.

\looseness -1 \textbf{Reward function.} A reward function $\TrueRew: \StateSpace \times \ActionSpace \times [\Horizon] \rightarrow [0,\Rmax]$ maps state-action-time step triplets to a reward. Given an \MDPnoR $\MDP$ and a reward function $\reward$, we denote the resulting MDP by $\MDP\cup \reward$.

\textbf{Value functions and optimality conditions.} We define the \textit{Q-function} $\Qfun{\MDP\cup\reward}{\Policy}{\timestep}(\state, \action)$ and \textit{value-function} $\Valfun{\MDP\cup\reward}{\Policy}{\timestep}(\state)$ of a policy $\Policy$ in the MDP $\MDP\cup\reward$ at time step $\timestep$, state $\state$ and action $\action$ as:

\resizebox{.95\linewidth}{!}{
  \begin{minipage}{\linewidth}
\begin{equation*}
    \Qfun{\MDP\cup\reward}{\Policy}{\timestep}(\state, \action) = \reward_\timestep(\state,\action) + \sum_{\state^\prime, \action^\prime} \Policy_{\timestep+1}(\action'|\state') \TransitionModel(\state^\prime|\state,\action) \Qfun{\MDP\cup\reward}{\Policy}{\timestep+1}(\state^\prime, \action^\prime); ~~ \Valfun{\MDP\cup\reward}{\Policy}{\timestep}(\state) = \sum_{\action} \Policy_\timestep(\action|\state) \Qfun{\MDP\cup\reward}{\Policy}{\timestep}(\state, \action)
\end{equation*}
\end{minipage}}

We define the \textit{advantage function} $A^{\Policy,\timestep}_{\MDP \cup \reward}(\state,\action)$ as $A^{\Policy,\timestep}_{\MDP\cup\reward}(\state,\action) = \Qfun{\MDP\cup\reward}{\Policy}{\timestep}(\state, \action) - \Valfun{\MDP\cup\reward}{\Policy}{\timestep}(\state)$. A policy $\Policy$ is optimal if $A^{\Policy,\timestep}_{\MDP\cup\reward}(\state,\action)$ is $\le 0$ for each time step $\timestep \in [\Horizon]$, state $\state\in \StateSpace$, action $\action \in \ActionSpace$.  We denote the set of optimal policies for the MDP $\MDP \cup \reward$ with $\Pi^*_{\MDP \cup \reward}$.

\textbf{State-visitation frequencies.} 
We define $\Occupancy{\MDP}{\Policy}{\timestep}{\timestep'}(\state|\state_0)$ as the probability of being in state $\state$ at time $\timestep' \geq \timestep$ following policy $\Policy$ in \MDPnoR $\MDP$ starting in state $\state_0$ at time $\timestep$. We can compute it recursively:
\[
\Occupancy{\MDP}{\Policy}{\timestep}{\timestep}(\state'|\state) \coloneqq \indicator{\state'=\state} ~\text{ and }~
\Occupancy{\MDP}{\Policy}{\timestep}{\timestep'+1}(\state'|\state)
\coloneqq \sum_{\state'',\tilde{\action}} \TransitionModel(\state' | \state'', \tilde{\action}) \Policy_{\timestep'}(\tilde{\action}|\state'') \Occupancy{\MDP}{\Policy}{\timestep}{\timestep'}(\state''|\state).
\]
We can define the visitation frequencies for state-action pairs analogously (see \Cref{app:simulation_lemma}).

\section{Active Learning for Inverse Reinforcement Learning (Active IRL)}
\label{sec:active-learning}

In this section, we first introduce the Active Inverse Reinforcement Learning problem with and without a generative model (\Cref{sec:problem-definition}). Then, we define the feasible reward set for finite-horizon MDPs (\Cref{sec:feasible-rewards}) and characterize the error propagation on the reward function and the value function (\Cref{sec:error-propagation}), extending results by \citet{metelli2021provably} to the finite horizon setting.

\subsection{Problem Definition}
\label{sec:problem-definition}
Our goal is to design an exploration strategy to construct a dataset of demonstrations $\mathcal{D}$ such that an arbitrary IRL algorithm can recover a \textit{good} reward function from it. To be agnostic to the choice of IRL algorithm, we consider the set of all feasible reward functions for a specific expert policy. Formally, we consider IRL problems $(\MDP, \ExpertPolicy)$ consisting of an \MDPnoR and an expert policy $\ExpertPolicy$, and we define the feasible reward set as follows.
\begin{definition}[Feasible Reward Set]\label{def:feasible_reward_set}
A reward function $\reward$ is \emph{feasible} for an IRL problem $(\MDP, \ExpertPolicy)$, if and only if the expert policy $\ExpertPolicy$ is optimal in $\MDP\cup\reward$. We call the set of all feasible reward functions $\mathcal{R}_{\MDP\cup\ExpertPolicy}$ the \emph{feasible reward set}. If we estimate the transition model and expert policy from samples, we refer to the \emph{recovered feasible set} $\RecoveredFeasibleSet = \mathcal{R}_{\EstMDP\cup\hat{\Policy}^E}$ in contrast to the \emph{exact feasible set} $\ExactFeasibleSet = \mathcal{R}_{\MDP\cup\ExpertPolicy}$.
\end{definition}
Now, we can formalize the goal of Active IRL as finding a exploration strategy that satisfies the following PAC optimality criterion.
\begin{restatable}[Optimality Criterion]{definition}{OptimalityCriterion}\label{def:correct}
Let $\ExactFeasibleSet$ be the exact feasible set and $\RecoveredFeasibleSet$ be the feasible set recovered after observing $n \geq 0$ samples collected from $\MDP$ and $\ExpertPolicy$. We say that an algorithm for Active IRL is \emph{$(\epsilon,\delta,n)$-correct} if after $n$ iterations with probability at least $1-\delta$ it holds that:
\begin{align*}
    \inf_{\hat{\reward} \in \RecoveredFeasibleSet} \sup_{\hat{\Policy}^* \in \Pi^*_{\EstMDP\cup\hat{\reward}}} \max_{\action} \Bigl| \Qfun{\MDP \cup \reward}{\Policy^*}{0}(\state_0,\action) - \Qfun{\MDP \cup \reward}{\hat{\Policy}^*}{0}(\state_0,\action) \Bigr| &\leq \epsilon \quad \text{for each } \reward \in \ExactFeasibleSet, \\
    \inf_{\reward \in \ExactFeasibleSet} \sup_{\Policy^* \in \Pi^*_{\MDP\cup\reward}} \max_{\action} \Bigl| \Qfun{\MDP \cup \reward}{\Policy^*}{0}(\state_0,\action) - \Qfun{\MDP \cup \reward}{\hat{\Policy}^*}{0}(\state_0,\action) \Bigr| &\leq \epsilon \quad \text{for each } \hat{\reward} \in \RecoveredFeasibleSet,
\end{align*}
where $\Policy^*$ is an optimal policy in $\MDP\cup\reward$ and $\hat{\Policy}^*$ is an optimal policy in $\EstMDP \cup \hat{\reward}$.
\end{restatable}
\looseness -1 The first condition states that for each reward in the exact feasible set, the best reward we could estimate in the recovered feasible set has a low error everywhere. This condition is a type of ``recall'': every possible true reward function needs to be captured by the recovered feasible set. The second condition ensures that there is a possible true reward function with a low error for every possible recovered reward function. This avoids an unnecessarily large recovered feasible set. This condition is a type of ``precision'': if we recover a reward function, it has to be close to a possible true reward function.
Note, that \citet{metelli2021provably} consider a similar optimality criterion  in their Definition 5.1. However, they consider a known target environment; hence, our \Cref{def:correct} is a stronger requirement.

\subsection{Feasible Rewards in Finite-horizon MDPs}
\label{sec:feasible-rewards}

\looseness -1 \citet{ng2000algorithms} characterize the feasible reward set implicitly in the infinite horizon setting, whereas \citet{metelli2021provably} characterize it explicitly. Here, we provide similar results for a finite horizon.

\begin{restatable}[Feasible Reward Set Implicit]{lemma}{LemmaIfh}\label{thm:ifh}\label{lemma:ifh}
A reward function $\reward$ is feasible if and only if for all $\state,\action,\timestep$ it holds that:
$A^{\Policy,\timestep}_{\MDP\cup\reward}(\state,\action) = 0$ if $\ExpertPolicy_\timestep(\action|\state) \ge 0$ and $A^{\Policy,\timestep}_{\MDP\cup\reward}(\state,\action) \leq 0$ if $\ExpertPolicy_\timestep(\action|\state) = 0$.
Moreover, if the second inequality is strict, $\ExpertPolicy$ is uniquely optimal, i.e., $\Pi^*_{\MDP\cup\reward} = \{\ExpertPolicy\}$.
\end{restatable}
\begin{restatable}[Feasible Reward Set Explicit]{lemma}{LemmaEfh}\label{lemma:efh}
A reward function $\reward$ is feasible if and only if there exists an $\{A_\timestep \in \reals_{\ge 0}^{\StateSpace\times\ActionSpace}\}_{\timestep \in [\Horizon]}$ and $\{V_\timestep \in \reals^{\StateSpace}\}_{\timestep\in[\Horizon]}$ such that for all $\state,\action,\timestep$ it holds that:
\begin{align*}
\reward_\timestep(\state, \action) = \textcolor{BlueViolet}{-A_\timestep(\state, \action)\indicator{\ExpertPolicy_\timestep(\action|\state) = 0}} + \textcolor{mygreen}{ V_\timestep(\state) + \sum_{\state'} \TransitionModel(\state'|\state,\action)V_{\timestep+1}(\state')}
\end{align*}
\end{restatable}
Here, the \textcolor{BlueViolet}{\textbf{first term}} ensures there is an advantage function for $\ExpertPolicy$ and it is $0$ for actions the expert takes and $A_h(s,a)$ for actions the expert does not take. The \textcolor{mygreen}{\textbf{second term}} corresponds to reward-shaping by the value function.

\subsection{Error Propagation}
\label{sec:error-propagation}
Next, we study the error propagation of estimating the transition model $\TransitionModel$ with $\EstTransitionModel$ and the expert policy $\ExpertPolicy$ with $\hat{\Policy}^E$. In particular, we bound the estimation error on the reward as a function of the estimation errors of $\EstTransitionModel$ and $\hat{\Policy}^E$, extending a result by \citet{metelli2021provably} to the finite horizon setting.
\begin{restatable}[Error Propagation]{theorem}{ThmError}\label{thm:error_propagation}
Let $(\MDP, \ExpertPolicy)$ and $(\widehat{\MDP}, \widehat{\Policy}^E)$ be two IRL problems. Then, for any $\reward \in \mathcal{R}_{(\MDP, \ExpertPolicy)}$ there exists $\widehat{\reward} \in \widehat{\mathcal{R}}_{(\widehat{\MDP}, \widehat{\Policy}^E)}$ such that:
\begin{align*}
|\reward_\timestep(\state,\action) - \widehat{\reward}_\timestep(\state,\action)| \le A_\timestep(\state,\action)|\ExpertPolicy_\timestep(\action|\state) - \widehat{\Policy}^E_\timestep(\action|\state)| + \sum_{\state'} V_{\timestep+1}(\state') |\TransitionModel(\state'|\state,\action)-\widehat{\TransitionModel}(\state'|\state,\action)|
\end{align*}
and we can bound $V_\timestep \le (\Horizon-\timestep) \Rmax$ and $A_\timestep \le (\Horizon-\timestep) \Rmax$.
\end{restatable}

In IRL, we cannot hope to recover the expert's reward function perfectly.
Instead, we aim to estimate a reward function that leads to an optimal policy with performance close to the expert's policy under the (unknown) real reward function. For example, suppose a specific state $\state$ is difficult to reach in the environment. In that case, the error on the reward function $\reward(\state,\cdot)$ will not impact the performance of the induced policy much. Formally, we are interested in studying the error propagation to the optimal value function. The next lemma will be crucial for analyzing this.
\begin{restatable}{lemma}{SimLemError}
\label{simulation_lemma_same_reward_different_policies}
Let $\MDP$ be an \MDPnoR, $\reward, \hat{\reward}$ two reward functions with optimal policies $\Policy^*, \hat{\Policy}^*$. Then,
\resizebox{\linewidth}{!}{
  \begin{minipage}{\linewidth}
\begin{align*}
\Qfun{\MDP\cup\reward}{\Policy^*}{\timestep}(\state, \action) - \Qfun{\MDP\cup\reward}{\hat{\Policy}^*}{\timestep}(\state, \action)
&\leq \sum_{\timestep'=\timestep}^{\Horizon} \sum_{\state', \action'} \left(\Occupancy{\MDP}{\Policy^*}{\timestep}{\timestep'}(\state', \action' | \state, \action) - \Occupancy{\MDP}{\hat{\Policy}^*}{\timestep}{\timestep'}(\state', \action' | \state, \action)\right) \left( \reward_{\timestep'}(\state', \action') - \hat{\reward}_{\timestep'}(\state', \action') \right)
\end{align*}
\end{minipage}}
\end{restatable}
\label{sec:error-propagation-value-function}
\looseness -1 By combining this lemma with \Cref{thm:error_propagation}, we can decompose the error in the value function and Q-function into the error in estimating the transition model and the error in estimating the expert policy.

\section{Recovering Feasible Rewards with a Generative Model}
\label{sec:sample-complexity-generative}

As a warmup, let us first study the sample complexity of a simple \textit{uniform sampling} strategy with access to the generative model of $\MDP$. We assume we can query a generative model about a state-action pair $(\state,\action)$ to receive a next state $\state^\prime \sim \TransitionModel(\cdot|\state,\action)$ and an expert action $\action_E \sim \pi^E(\cdot|\state)$. This allows us to introduce key ideas and serves as a baseline to compare later results to. We adapt the infinite-horizon results by \citet{metelli2021provably} to the finite-horizon setting, and our stronger PAC requirement in \Cref{def:correct}. We first discuss how we can estimate the transition model and the policy (\Cref{sec:estimation}) before stating the sample complexity of the uniform sampling strategy (\Cref{sec:sample-complexity-generative-sec}).

\subsection{Estimating Transition Model and Expert Policy}
\label{sec:estimation}
In each iteration $\episode$, let $n_\episode^\timestep(\state,\action,\state^\prime)$ be the number of times we observed the transitions $(\state,\action,\state)$ at time $\timestep$ up to iteration $\episode$. Also, we define $n_\episode^\timestep(\state,\action) = \sum_{\state^\prime} n_\episode^\timestep(\state,\action,\state^\prime)$, and $n_\episode^\timestep(\state) = \sum_{\action} n_\episode^\timestep(\state,\action)$. Then we can estimate the transition model and expert policy by
\[
\EstTransitionModel_\episode(\state^\prime|\state,\action) = \frac{\sum_{\timestep=1}^\Horizon n_\episode^\timestep(\state,\action,\state^\prime)}{\max(1, \sum_{\timestep=1}^\Horizon n_\episode^\timestep(\state,\action))} \qquad\qquad
\hat{\Policy}^E_{\episode,\timestep}(\action|\state) = \frac{n_\episode^\timestep(\state,\action)}{\max(1, n_\episode^\timestep(\state))}.
\]
\looseness -1 In \Cref{app:uniform_sampling} we derive Hoeffding's confidence intervals for the transition model and the expert policy. Combining these with \Cref{thm:error_propagation}, we can compute the uncertainty on the recovered reward as:
\[
\rewardci^\timestep_\episode(\state, \action) = (\Horizon-\timestep) \Rmax \min \left(1, 2 \sqrt{\frac{2 \ell^\timestep_\episode(\state, \action)}{{n^\timestep_\episode}(\state,\action)}} \right),
\]
where $\ell^\timestep_\episode(\state, \action) =  \log\left(24 \StateSpaceSize \ActionSpaceSize \Horizon ({n^\timestep_\episode}(\state, \action))^2 / \delta \right)$. We can show that for any pair of reward functions $\reward\in\ExactFeasibleSet$ and $\hat{\reward}\in\RecoveredFeasibleSet$, the difference $|\reward_\timestep(\state,\action)-\hat{\reward}_{\episode,\timestep}(\state,\action)| \leq \rewardci^\timestep_\episode(\state,\action)$. This uncertainty estimate will be a key component in all of our theoretical analysis.

\subsection{Uniform Sampling Strategy}
\label{sec:sample-complexity-generative-sec}
In each iteration $\episode$, the \textit{uniform sampling} strategy allocates $n_{\max}$ samples uniformly over $[\Horizon] \times \StateSpace \times \ActionSpace$. It estimates the reward uncertainty and stops as soon as $\Horizon \max_{\timestep,\state,\action} \rewardci_\episode^\timestep(\state,\action) \le \epsilon$.
The next theorem characterizes the sample complexity of uniform sampling with a generative model.
\begin{theorem}[Sample Complexity of Uniform Sampling IRL]
The uniform sampling strategy fulfills \Cref{def:correct} with a number of samples upper bounded by:
\[
n \leq \tilde{\mathcal{O}} \Bigl( \Horizon^5 \Rmax^2 \StateSpaceSize \ActionSpaceSize / \epsilon^2 \Bigr),
\]
where $\mathcal{O}$ suppresses logarithmic terms.
\end{theorem}
This sample complexity bound appears slightly worse than the one in \citet{metelli2021provably}, who find $(1-\gamma)^{-4}$ which would translate to $\Horizon^4$. This is, however, due to the fact that we consider reward functions that can depend on the timestep $\timestep$. If we assume the reward function does not depend on $\timestep$, we gain a factor of $\Horizon$, obtaining the same result as \citet{metelli2021provably}.

\section{Active Exploration for Inverse Reinforcement Learning}
\label{sec:aceirl}

Let us now turn to our original problem: recovering the expert's reward function in an unknown environment {\em without} a generative model. This problem is harder since we need to create an exploration strategy to acquire the desired information about the environment. We now propose a novel sample-efficient exploration algorithm for IRL that we call \AlgNameDef. The algorithm takes inspiration from recent works on reward-free exploration \citep{kaufmann2021adaptive} and exploration strategies in RL \citep{auer2008near}. We divide the explanation of the algorithm into two parts. First, we introduce a simplified version of the algorithm, which comes with a problem independent sample complexity result (\Cref{sec:stopping-condition}). Next, we introduce the full algorithm, which considers the expected reduction of uncertainty in the next iteration to improve exploration and maintains a confidence set of plausibly optimal policies to focus on the most relevant regions (\Cref{sec:stopping-condition-dep}). The full algorithm provides a tighter, problem-dependent sample complexity bound (\Cref{sec:theoretical_results}). \Cref{alg:aceirl} contains pseudo-code of \AlgNameShort, and \Cref{app:theory} contains the detailed theoretical analysis including proofs of all results discussed here.

\begin{algorithm}[t]
\caption{\AlgNameShort algorithm for IRL in an unknown environment.}
\label{alg:aceirl}
\begin{algorithmic}[1]
    \State \textbf{Input:} significance $\delta \in (0,1)$, target accuracy $\epsilon$, IRL algorithm $\irlalg$, number of episodes $N_E$
    \State Initialize $\episode \gets 0$, \hspace{0.2em} $\epsilon_0 \gets \Horizon/10$
    \While{$\epsilon_\episode > \epsilon/4$}
        \State Solve (convex) optimization problem (\ref{eq:policy_optimization}) to obtain $\Policy_\episode$
        \State Explore with policy $\Policy_\episode$ for $N_E$ episodes, observing transitions and expert actions
        \State $\episode \gets \episode + 1$
        \State Update $\EstTransitionModel_{\episode}$, $\hat{\Policy}_{\episode}$, $\rewardci_\episode^\timestep$, and $\hat{\reward}_\episode \gets \irlalg(\RecoveredFeasibleSet)$
        \State Update accuracy $\epsilon_\episode \gets \max_\action \hat{E}_\episode^0(\state_0, \action)$
    \EndWhile
    \State \Return Estimated reward function $\hat{\reward}_\episode$
\end{algorithmic}
\end{algorithm}

\subsection{Uncertainty-based Exploration for IRL}
\label{sec:stopping-condition}

The first idea of \AlgNameShort is similar to reward-free UCRL~\citep{kaufmann2021adaptive}. Our goal is to fulfill the PAC requirement in \Cref{def:correct}. Hence, we start from an upper bound on the estimation error between the performance of the optimal policy $\hat{\Policy}^*$ for a reward $\hat{\reward}\in\RecoveredFeasibleSet$ in the recovered feasible set and the optimal policy $\Policy^*$ for a reward function $\reward \in \ExactFeasibleSet$ in the true MDP $\MDP$. We will then use this upper bound to drive the exploration.
For each timestep $\timestep$ and iteration $\episode$, we define the error:
\begin{equation}\label{eq:error}
\hat{e}_{\episode}^{\timestep}(\state,\action; \Policy^*, \hat{\Policy}^*) = \Bigl| \Qfun{\MDP\cup\reward}{\Policy^*}{\timestep}(\state, \action) - \Qfun{\MDP\cup\reward}{\hat{\Policy}^*}{\timestep}(\state, \action) \Bigr|.
\end{equation}
We can define an upper bound on these errors recursively with $C^\Horizon_\episode(\state,\action) = 0$ and
\begin{equation}
\label{eq:error1}
    E_\episode^\timestep(\state, \action) = \min\Bigl((\Horizon-\timestep)\Rmax, \rewardci_\episode^\timestep(\state, \action) + \sum_{\state^\prime} \EstTransitionModel(\state^\prime | \state, \action) \max_{\action^\prime \in \ActionSpace}  E_\episode^{\timestep+1}(\state^\prime, \action^\prime) \Bigr). \tag{EB1}
\end{equation}
It is straightforward to show that $\hat{e}_{\episode}^{\timestep}(\state,\action; \Policy^*, \hat{\Policy}^*) \leq E_\episode^\timestep(\state, \action)$ for any two policies $\Policy^*, \hat{\Policy}^*$.
Using this error bound, we can introduce a simplified version of \AlgNameShort that explores greedily with respect to $E_\episode^\timestep(\state, \action)$. We call this algorithm ``\AlgNameShort Greedy''. Note that this is equivalent to solving the RL problem defined by $\MDP \cup \rewardci^\timestep_\episode$; hence, we can use any RL solver to find the exploration policy in practice.
If we explore with this greedy policy, we can stop if:
\begin{align}
\label{eq:stopping-condition-simple}
    4 \max_\action E_\episode^0(\state_0, \action) \leq \epsilon. \tag{SP1}
\end{align}
We can show that when this stopping condition holds, the solution fulfills the PAC requirement \ref{def:correct}. Furthermore, we show in \Cref{app:sample-complexity} that \AlgNameShort Greedy achieves a sample complexity on order $\tilde{\mathcal{O}} \left( \Horizon^5 \Rmax^2 \StateSpaceSize \ActionSpaceSize / \epsilon^2 \right)$, which matches the sample complexity of uniform sampling \emph{with a generative model}. This is already a strong result implying that we do not need a generative model to achieve a good sample complexity for IRL. However, it turns out we can improve the algorithm further.

\subsection{Problem Dependent Exploration}
\label{sec:stopping-condition-dep}

\looseness -1 \AlgNameShort Greedy is limited in two ways: (i) it explores states that have high uncertainty so far, whereas our goal is to reduce uncertainty \emph{in the next iteration}, and (ii) it explores to reduce the uncertainty about all policies, whereas our goal is to reduce the uncertainty primarily about \emph{plausibly optimal} policies. To address these limitations, we propose two modifications that yield the full \AlgNameShort algorithm.

\textbf{Reducing future uncertainty.}
The greedy policy w.r.t. $E_\episode^\timestep$ explores states in which the estimation error on the Q-functions is large. However, note that this is not exactly what we want, namely, to reduce the uncertainty the most. In particular, if we explore for more than one episode before updating the exploration policy, we should choose an exploration policy that considers how the uncertainty will reduce during exploration. Ideally, we would explore with a policy that minimizes $E_{\episode+1}^\timestep$. However, we cannot compute this quantity exactly. Instead, we can approximate it using our current estimate of the transition model. Concretely, if we have an exploration policy $\Policy$, we can estimate the reward uncertainty at the next iteration as:
\[
\hat{\rewardci}^\timestep_{\episode+1}(\state, \action) = (\Horizon-\timestep) \Rmax \min \left(1, 2 \sqrt{\frac{2 \ell^\timestep_\episode(\state, \action)}{{n^\timestep_\episode}(\state,\action) + {\hat{n}^\timestep_\Policy}(\state,\action)}} \right),
\]
where ${\hat{n}^\timestep_\Policy}(\state,\action) = N_E \cdot \Occupancy{\MDP}{\Policy}{0}{\timestep}(\state, \action | \state_0)$ is the expected number of times $\Policy$ visits $(\state,\action)$ at time $\timestep$ and $N_E$ is the number of episodes we will explore with $\Policy$. We can use this estimate to find a policy that minimizes our estimate of $E_\episode^{\timestep+1}$. While our original approach was akin to ``uncertainty sampling'', we now have a better way to measure the ``informativeness'' of choosing an exploration policy. This is a common pattern when designing query strategies in active learning \citep{settles2009active}. Note, that this argument does not rely on the IRL problem and can be used to independently improve algorithms for reward-free exploration (cf. \Cref{app:reward-free-exploration}).

\textbf{Focusing on plausibly optimal policies.}
By exploring greedily w.r.t. $E^\timestep_\episode$, we reduce the estimation error of all policies. However, we are primarily interested in estimating the distance between policies $\Policy^* \in \Pi^*_{\MDP\cup\reward}$ and $\hat{\Policy}^* \in \Pi^*_{\MDP\cup\hat{\reward}}$ with $\reward\in\ExactFeasibleSet$ and $\hat{\reward}\in\RecoveredFeasibleSet$.
Of course, we do not know these sets, so we cannot use them directly to target the exploration. Instead, assume we can construct a set of plausibly optimal policies $\policyci_\episode$ that contains all $\Policy^*$ and $\hat{\Policy}^*_\episode$ with high probability. Then, we can redefine our upper bounds recursively as $\hat{E}_\episode^\Horizon(\state, \action) = 0$ and:
\begin{align}
\label{eq:error2}
    \hat{E}_\episode^\timestep(\state, \action) = \min\Bigl((\Horizon-\timestep)\Rmax, \rewardci_\episode^\timestep(\state, \action) + \sum_{\state'} \EstTransitionModel(\state' | \state, \action) \max_{\Policy\in\policyci_{\episode-1}} \Policy(\action'|\state') \hat{E}_\episode^{\timestep+1}(\state', \action') \Bigr), \tag{EB2}
\end{align}
\looseness -1 In contrast to (\ref{eq:error1}), we maximize over policies in $\policyci_\episode$ rather than all actions. Acting greedily with respect to $\hat{E}_\episode^\timestep(\state, \action)$ is equivalent to finding the optimal policy $\Policy_\episode \in \hat{\Pi}_k$ for the RL problem defined by $\MDP \cup C^h_k$.
To construct the set of plausibly optimal policies, we use an arbitrary IRL algorithm $\irlalg$. We only assume that $\irlalg$ will return a reward function  $\hat{\reward}_\episode\in\RecoveredFeasibleSet$. Then, we can construct a set of plausibly optimal policies as $\policyci_\episode = \{ \Policy | \Valfun{\EstMDP\cup{\hat{\reward}_\episode}}{*}{}(\state_0) - \Valfun{\EstMDP\cup{\hat{\reward}_\episode}}{\Policy}{}(\state_0) \leq 10 \epsilon_\episode \}$. We show in \Cref{app:sample-complexity2} that $\policyci_\episode$ contains both $\Policy^*$ and $\hat{\Policy}^*_\episode$ with high probability. This choice is based on ideas by \citet{zanette2019}.

We can define a stopping condition analogously to (\ref{eq:stopping-condition-simple}):
\begin{align}
\label{eq:stopping-condition}
    4 \max_\action \hat{E}_\episode^0(\state_0, \action) \leq \epsilon. \tag{SP2}
\end{align}
Again, we can prove that if the algorithms stops due to (\ref{eq:stopping-condition}), then $\RecoveredFeasibleSet$ respects \Cref{def:correct}.

\textbf{Implementing \AlgNameShort.}
To implement the full algorithm, we need to solve an optimization problem:
\begin{align}\label{eq:policy_optimization}
\Policy_\episode \in \mathop\argmin_\Policy \max_{\hat{\Policy}\in\policyci_{\episode-1}} \hat{E}_{\episode+1}^0(\state_0, \hat{\Policy}(\state_0)) \tag{ACE}
\end{align}
The solution to this problem is the exploration policy that minimizes the uncertainty at the next iteration about plausibly optimal policies. This combines both modifications we just discussed. This problem might seem difficult to solve at first, but, perhaps surprisingly, it can be formulated as a convex optimization problem solvable with standard techniques (cf. \Cref{app:optimization_problem}).

\subsection{Sample Complexity of \AlgNameShort}\label{sec:theoretical_results}

In this section, we present our main result about the sample complexity of \AlgNameShort. The result is problem-dependent, and, in particular, depends on the advantage function $A_{\MDP\cup\reward}^{*,\timestep}(\state, \action)$, where $\reward$ is the reward function in the exact feasible set $\ExactFeasibleSet$ closest to the
reward function $\hat{r}_k$ which belongs to the estimated feasible set $\RecoveredFeasibleSet$. The advantage function $A_{\MDP\cup\reward}^{*,\timestep}(\state, \action)$ acts similarly to a suboptimality gap: the closer the value of the second best action is to the best action, the harder it is to identify the best action and infer the correct reward function.
\begin{restatable}{theorem}{SampleComplexityProblemDependent}[\AlgNameShort Sample Complexity]
\AlgNameShort returns a ($\epsilon$, $\delta$, $n$)-correct solution with
\[
n \leq \tilde{\mathcal{O}}\left( \min\left[ 
\frac{\Horizon^5 \Rmax^2 \StateSpaceSize \ActionSpaceSize}{\epsilon^2},
\frac{\Horizon^4 \Rmax^2 \StateSpaceSize \ActionSpaceSize \epsilon_{\tau-1}^2}{\min_{\state,\action,\timestep} (A_{\MDP\cup\reward}^{*,\timestep}(\state, \action))^2 \epsilon^2}
\right] \right)
\]
where $\epsilon_{\tau-1}$ depends on the choice of $N_E$, the number of episodes of exploration in each iteration. $A_{\MDP\cup\reward}^{*,\timestep}(\state, \action)$ is the advantage function of $\reward \in \argmin_{\reward\in\ExactFeasibleSet} \max_{\timestep,\state,\action} (\reward_\timestep(\state,\action) - \hat{\reward}_{\episode,\timestep}(\state,\action))$, the reward function from the feasible set $\ExactFeasibleSet$ closest to the estimated reward function $\hat{\reward}_\episode$.
\end{restatable}
This result is the minimum of two terms. The first term is problem independent and it is achieved both by \AlgNameShort Greedy and the full \AlgNameShort. This bound matches the bound we saw previously with a generative model. Hence, \AlgNameShort achieves the same results without access to the generative model.
Using (\ref{eq:policy_optimization}) can yield a better sample complexity, represented by the second term in the minimum. This bound depends on two main components: the ratio $\epsilon_{\tau-1}/\epsilon$ and the advantage function $A_{\MDP\cup\reward}^{*,\timestep}(\state, \action)$. The ratio depends on the choice of $N_E$, the number of exploration episodes per iteration. If $N_E$ is small, then the $\epsilon$-ratio will be also small. If $N_E$ is large the algorithm will perform similarly to a uniform sampling strategy. \Cref{app:sample-complexity2} provides the full proof of this theorem.

\begin{wrapfigure}{R}{0.45\textwidth}\vspace*{-1em}
\centering
\includegraphics[width=\linewidth]{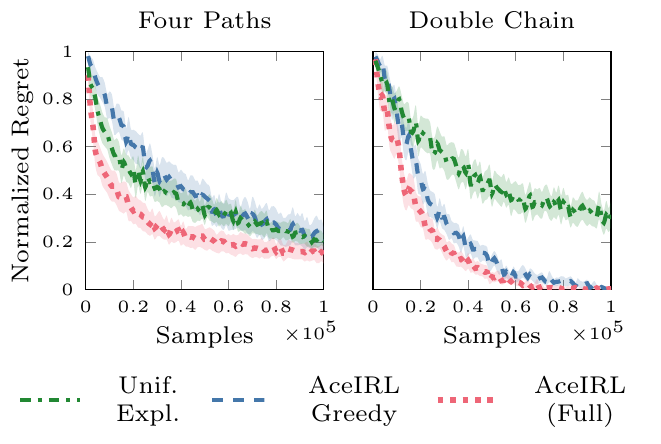}\vspace*{-1.5em}
\caption{Normalized regret (lower is better) of the policy optimizing for the inferred reward in the estimated MDP as a function of the number of samples. The plots show the mean and $95\%$ confidence intervals computed using 50 random seeds. We use $N_E = 50$.}\label{fig:results}\vspace*{-0.5em}
\end{wrapfigure}

\begin{table}[t]
    \newcommand{\result}[2]{$#1 \pm #2$}
    \newcommand{\hresult}[2]{$\mathbf{#1 \pm #2}$}
    \newcommand{\todoresult}[2]{{\color{Red}TODO}}
    \centering
    \resizebox{\linewidth}{!}{\begin{tabular}{lcccccc}
        \toprule
        & \specialcell{Uniform sampling\\(gener. model)} & \specialcell{TRAVEL (gener. model) \\ \citep{metelli2021provably}} & \specialcell{Random \\ Exploration} & \specialcell{\AlgNameShort \\ Greedy} & \specialcell{\AlgNameShort \\ (Full)} \\
        \midrule
        Four Paths (\Cref{fig:motivating_example}) & \result{1900}{71} & & \result{17840}{1886} & & & \\
        \hspace{1em}-- $N_E=50$ & & \result{1560}{76} & & \result{24180}{1747} & \hresult{10780}{1369} \\
        \hspace{1em}-- $N_E=100$ & & \result{2000}{0} & & \result{32760}{2172} & \result{14080}{1603} \\
        \hspace{1em}-- $N_E=200$ & & \result{4000}{0} & & \result{52000}{4057} & \result{16160}{2033} \\
        \midrule
        \specialcell{Double Chain \\ \citep{kaufmann2021adaptive}} & \result{1980}{66} &  & \result{23640}{2195} &  &  \\
        \hspace{1em}-- $N_E=50$ & & \result{1120}{46} &  & \result{16240}{842} & \hresult{11580}{870} \\
        \hspace{1em}-- $N_E=100$ & & \result{2000}{0} & & \result{22200}{1329} & \result{15440}{1031} \\
        \hspace{1em}-- $N_E=200$ & & \result{4000}{0} & & \result{37200}{1664} & \result{20400}{1629} \\
        \midrule
        \citet{metelli2021provably}: & & & & & & \\
        Random MDPs ($N_E=1$) & \result{22}{1} & \result{27}{1} & \hresult{22}{1} & \result{23}{1} & \hresult{21}{1} \\
        Chain ($N_E=1$) & \result{78}{2} & \result{76}{4} & \result{161}{8} & \result{153}{8} & \hresult{142}{9} \\
        Gridworld ($N_E=1$) & \result{43}{2} & \result{35}{2} & \hresult{45}{2} & \hresult{46}{3} & \hresult{48}{2} \\
        \bottomrule
    \end{tabular}}
    \caption{Sample complexity of \AlgNameShort compared to random exploration and methods that use a generative model. We show the number of samples necessary to find a policy with normalized regret less than $0.4$. We report means and standard errors computed over $50$ random seeds each. For each environment, we highlight in \textbf{bold} the method that achieves the best performance without access to a generative model. If multiple methods are within one standard error distance, we highlight all of them. Overall, \AlgNameShort is the most sample efficient method without a generative model if $N_E$ is chosen small enough. In \Cref{app:additional_results}, we show learning curves for all individual experiments.}\label{tab:results}
    \vspace*{-1.5em}
\end{table}

\newpage %

\section{Experiments}\label{sec:experiments}

We perform a series of simulation experiments to evaluate \AlgNameShort. We simulate a (deterministic) expert policy using an underlying true reward function, and compare it to the recovered reward functions.
We provide a code to reproduce all of our experiments at \codelink.

Our main evaluation metric is a \emph{normalized regret}:
\[
\bigl(\Valfun{\MDP\cup\reward}{\Policy^*}{0}(\state_0) - \Valfun{\MDP\cup\reward}{\hat{\Policy}^*}{0}(\state_0)\bigr)  / \bigl(\Valfun{\MDP\cup\reward}{\Policy^*}{0}(\state_0) - \Valfun{\MDP\cup\reward}{\bar{\Policy}^*}{0}(\state_0)\bigr),
\]
where $\Policy^*$ is the optimal policy for $\MDP\cup\reward$, $\hat{\Policy}^*$ is the optimal policy for $\EstMDP\cup\hat{\reward}$, and $\bar{\Policy}^*$ is the worst possible policy for $\reward$, i.e., the optimal policy for $\MDP\cup(-\reward)$.

We introduce the \emph{Four Paths} environment shown in \Cref{fig:motivating_example}, which consists of four chains of states that have different randomly sampled transition probabilities. One path has a goal with reward $1$; all other rewards are $0$. We also evaluate on \emph{Random MDPs} with uniformly sampled transition dynamics and reward functions, the \emph{Double Chain} environment proposed by \citet{kaufmann2021adaptive}, and the \emph{Chain} and \emph{Gridworld} environments proposed by \citet{metelli2021provably}. \Cref{app:environment_details} provides details on the transition dynamics and rewards of all environments. 

\looseness -1 We compare \AlgNameShort and \AlgNameShort Greedy to a uniformly random exploration policy, as a naive exploration strategy. Further, we consider uniform sampling with a generative model as well as TRAVEL \citep{metelli2021provably}, which can be more sample efficient because they do not need to explore the environment. Note that TRAVEL is designed to learn a reward to be transferred to a known target environment. Instead, we use a modified version that uses the estimated MDP as a target.

\looseness -1 \Cref{tab:results} shows the sample efficiency of all algorithms in all environments, measured as the number of samples needed to achieve a regret threshold of $0.4$ (different thresholds yield similar conclusions; cf. \Cref{app:experiment_details}). \AlgNameShort is the most sample efficient exploration strategy without access to a generative model; but the relative differences between the methods depend on the environment. In some cases, \AlgNameShort even performs comparably to methods using a generative model, as the theory predicts.

In the \emph{Four Paths} and \emph{Double Chain} environments, we also vary the $N_E$ parameter. \AlgNameShort performs better for small values at the computational cost of updating the exploration policy more often. If $N_E$ is too large, using \AlgNameShort can be as bad as a uniformly random exploration policy. Increasing $N_E$ hurts the performance of \AlgNameShort Greedy more severely, which does not consider $N_E$ explicitly. \Cref{fig:results} shows the normalized regret as a function of the number of samples in \emph{Four Paths} and \emph{Double Chain}. In both cases \AlgNameShort performs best. However, \AlgNameShort Greedy is worse than random exploration in the \emph{Four Paths} environment. Hence, we find that the problem dependent exploration strategy of the full algorithm significantly improves the sample efficiency.

\section{Conclusion}
\looseness -1 We considered active inverse reinforcement learning (IRL) with unknown transition dynamics and expert policy and introduced \AlgNameShort, an efficient exploration strategy to learn about both the dynamic and the expert policy with the goal of inferring the reward function as efficiently as possible.

\looseness -1 Our approach is a crucial step towards IRL algorithms with theoretical guarantees, but future work is needed to move to more practical settings. In particular, it would be interesting to extend the approach to continuous state and action spaces (e.g. using function approximation), and to obtain an efficient algorithm that does not require solving convex optimization problems. From a theoretical perspective, it would be useful to derive a lower bound on the sample complexity of the IRL problem, to understand if the IRL problem is inherently more difficult than usual RL. Beyond IRL, we believe that some of our methods could be useful for other settings, such as reward-free exploration (cf. \Cref{app:reward-free-exploration}).

\looseness -1 Sample efficient IRL is a promising way to apply RL in situations where there is no well-specified reward function available. Of course, even robust IRL algorithms pose a risk of misuse. But, we are optimistic that these methods will overall lead to safer RL systems that can be used in real applications.

\section*{Acknowledgements}

This project has received funding from the Microsoft Swiss Joint Research Center (Swiss JRC), Google Brain, and from the European Research Council (ERC) under the European Union's Horizon 2020 research and innovation programme grant agreement No 815943. We thank Bhavya Sukhija for valuable feedback on an early draft, and Guy Shacht for pointing out a mistake in an earlier version of this paper.

\newcommand{\ICML}{Proceedings of International Conference on Machine Learning (ICML)}
\newcommand{\RSS}{Proceedings of Robotics: Science and Systems (RSS)}
\newcommand{\NeurIPS}{Advances in Neural Information Processing Systems}
\newcommand{\IJCAI}{Proceedings of International Joint Conferences on Artificial Intelligence}
\newcommand{\ICLR}{International Conference on Learning Representations (ICLR)}
\newcommand{\CoRL}{Conference on Robot Learning (CoRL)}
\newcommand{\UAI}{Uncertainty in Artificial Intelligence}
\newcommand{\AAAI}{AAAI Conference on Artificial Intelligence}
\newcommand{\AISTATS}{International Conference on Artificial Intelligence and Statistics (AISTATS)}

\bibliographystyle{abbrvnat}
\bibliography{biblio}

\clearpage
\appendix\small

\setcounter{table}{0}
\renewcommand{\thetable}{\Alph{section}.\arabic{table}}
\setcounter{figure}{0}
\renewcommand{\thefigure}{\Alph{section}.\arabic{figure}}

\setcounter{theorem}{0}
\renewcommand{\thetheorem}{\Alph{section}.\arabic{theorem}}
\renewcommand{\thedefinition}{\Alph{section}.\arabic{theorem}}
\renewcommand{\thelemma}{\Alph{section}.\arabic{theorem}}
\renewcommand{\thecorollary}{\Alph{section}.\arabic{theorem}}

\renewcommand \thepart{}
\renewcommand \partname{}
\doparttoc
\faketableofcontents
\part{Appendix}
\parttoc

\section{Overview of Notation}

In \Cref{tab:notation}, we provide a reference of the notation and symbols used in our paper.

\begin{table}[h]
    \centering
    \caption{Overview of our notation}\label{tab:notation}
    \resizebox{.99\linewidth}{!}{
    \begin{tabular}{ccc}
        \toprule
        Symbol & Name & Signature  \\
        \midrule
        $\MDP$ & Markov decision process without reward (\MDPnoR) & $(\StateSpace, \ActionSpace, \TransitionModel, \Horizon, \state_0)$ \\
        $\StateSpace$ & State space & \\
        $\ActionSpace$ & Action space & \\
        $\TransitionModel$ & Transition model & $\StateSpace \times \ActionSpace \to \Delta_\StateSpace$ \\
        $\Horizon$ & Horizon & $\Horizon \in \mathbb{N}^+$ \\
        $\state_0$ & Initial state & $\state_0\in\StateSpace$ \\
        $\Policy$ & Policy & $\StateSpace\times [\Horizon] \to \Delta_\ActionSpace$ \\
        $\reward$ & Reward function & $\StateSpace\times\ActionSpace\times [\Horizon] \to [0, \Rmax]$, $\Rmax\in\reals^+$ \\
        $\MDP\cup\reward$ & Markov decision process (MDP) & $(\StateSpace, \ActionSpace, \TransitionModel, \Horizon, \state_0, \reward)$ \\
        $\Qfun{\MDP\cup\reward}{\Policy}{\timestep}$ & Q-function (of $\Policy$ in $\MDP\cup\reward$) & $\StateSpace\times\ActionSpace\times [\Horizon] \to \reals$ \\
        $\Valfun{\MDP\cup\reward}{\Policy}{\timestep}$ & Value function (of $\Policy$ in $\MDP\cup\reward$) & $\StateSpace\times [\Horizon] \to \reals$ \\
        $A^{\Policy,\timestep}_{\MDP \cup \reward}$ & Advantage function (of $\Policy$ in $\MDP\cup\reward$) & $\StateSpace\times\ActionSpace\times [\Horizon] \to \reals$ \\
        $\Occupancy{\MDP}{\Policy}{\timestep}{\cdot}(\cdot|\state_0)$ & \specialcell{State-visitation frequency \\ (conditioned on state)} & $[\Horizon] \to \Delta_\StateSpace$ \\
        $\Occupancy{\MDP}{\Policy}{\timestep}{\cdot}(\cdot|\state_0, \action_0)$ & \specialcell{State-visitation frequency \\ (conditioned on state-action)} & $[\Horizon] \to \Delta_\StateSpace$ \\
        $\Occupancy{\MDP}{\Policy}{\timestep}{\cdot}(\cdot,\cdot|\state_0)$ & \specialcell{State-action-visitation frequency \\ (conditioned on state)} & $[\Horizon]\times\StateSpace \to \Delta_\ActionSpace$ \\
        $\Occupancy{\MDP}{\Policy}{\timestep}{\cdot}(\cdot,\cdot|\state_0, \action_0)$ & \specialcell{State-action-visitation frequency \\ (conditioned on state)} & $[\Horizon]\times\StateSpace \to \Delta_\ActionSpace$ \\
        $\mathcal{R}_{\MDP\cup\reward}$ & Feasible set of $\MDP\cup\reward$ & \\
        $\ExactFeasibleSet = \mathcal{R}_{\MDP\cup\ExpertPolicy}$ & Exact feasible set & \\
        $\RecoveredFeasibleSet = \mathcal{R}_{\EstMDP\cup\hat{\Policy}^E}$ & Recovered feasible set & \\
        $\epsilon$ & Target accuracy & $\epsilon \in\reals^+$ \\
        $\delta$ & Significancy & $\delta \in (0,1)$ \\
        $N_E$ & Number of exploration episodes & $N_E \in \mathbb{N}^+$ \\
        \bottomrule
    \end{tabular}
    }
\end{table}

\section{Proofs of Theoretical Results}
\label{app:theory}

\subsection{Simulation Lemmas}
\label{app:simulation_lemma}

In this section, we establish several simulation lemmas that we will use throughout our analysis. Some of the results were already derived in prior work for the infinite horizon setting, e.g., by \citet{zanette2019} and \citet{metelli2021provably}. For completeness, we provide proofs for all results in the finite-horizon setting.

\begin{definition}[Occupancy measures]\label{def:occupancy}
We define $\Occupancy{\MDP}{\Policy}{\timestep}{\timestep'}(\state|\state_0)$ as the probability of being in state $\state$ at timestep $\timestep' \geq \timestep$ following a policy $\Policy$ in \MDPnoR $\MDP$ starting in state $\state_0$ at timestep $\timestep$. We can compute it recursively as:
\begin{align*}
&\Occupancy{\MDP}{\Policy}{\timestep}{\timestep}(\state'|\state) \coloneqq \indicator{\state'=\state} \\
&\Occupancy{\MDP}{\Policy}{\timestep}{\timestep'+1}(\state'|\state)
\coloneqq \sum_{\state'',\tilde{\action}} \TransitionModel(\state' | \state'', \tilde{\action}) \Policy_{\timestep'}(\tilde{\action}|\state'') \Occupancy{\MDP}{\Policy}{\timestep}{\timestep'}(\state''|\state)
\end{align*}

We define the same probability for state-action pairs analogously:
\begin{align*}
&\Occupancy{\MDP}{\Policy}{\timestep}{\timestep'}(\state', \action'|\state, \action) \coloneqq \indicator{\state'=\state, \action'=\action} \\
&\Occupancy{\MDP}{\Policy}{\timestep}{\timestep'+1}(\state', \action'|\state, \action) 
\coloneqq \sum_{\tilde{\state},\tilde{\action}} 
\Policy_{\timestep'}(\action'|\state') \TransitionModel(\state' | \tilde{\state}, \tilde{\action}) \Occupancy{\MDP}{\Policy}{\timestep}{\timestep'}(\tilde{\state}, \tilde{\action}|\state, \action)
\end{align*}
as well as
\begin{align*}
&\Occupancy{\MDP}{\Policy}{\timestep}{\timestep}(\state', \action'|\state) \coloneqq \Policy_\timestep(\action' | \state') \indicator{\state'=\state} \\
&\Occupancy{\MDP}{\Policy}{\timestep}{\timestep'+1}(\state', \action'|\state) 
\coloneqq \sum_{\tilde{\state},\tilde{\action}} 
\Policy_{\timestep'}(\action'|\state') \TransitionModel(\state' | \tilde{\state}, \tilde{\action}) \Occupancy{\MDP}{\Policy}{\timestep}{\timestep'}(\tilde{\state}, \tilde{\action}|\state)
\end{align*}
Because the environment is Markovian, it also holds for $\timestep' > \timestep$ that
\begin{align*}
\Occupancy{\MDP}{\Policy}{\timestep}{\timestep'}(\state'|\state) = \sum_{\tilde{\state}, \action} \Occupancy{\MDP}{\Policy}{\timestep+1}{\timestep'}(\state'|\tilde{\state}) \TransitionModel(\tilde{\state} | \state, \action) \Policy_\timestep(\action | \state)
\end{align*}
and equivalently for state-action pairs.
\end{definition}

\begin{lemma}\label{qfun_occupancy_lemma}
The value function and Q-function of a policy $\Policy$ in an MDP $\MDP \cup \reward$ at timestep $\timestep$ can be expressed as:
\begin{align*}
&\Valfun{\MDP\cup\reward}{\Policy}{\timestep}(\state) = \sum_{\timestep'=\timestep}^{\Horizon} \sum_{\state',\action'} \Occupancy{\MDP}{\Policy}{\timestep}{\timestep'}(\state', \action'|\state) \reward_{\timestep'}(\state',\action') \\
&\Qfun{\MDP\cup\reward}{\Policy}{\timestep}(\state, \action) = \sum_{\timestep'=\timestep}^{\Horizon} \sum_{\state',\action'} \Occupancy{\MDP}{\Policy}{\timestep}{\timestep'}(\state', \action'|\state, \action) \reward_{\timestep'}(\state',\action')
\end{align*}
\end{lemma}

\begin{proof}
We show the result for the value function; the derivation for the Q-function is analogous.

Note that for $\timestep = \Horizon$ the statement holds because $\Valfun{\MDP\cup\reward}{\Policy}{\Horizon}(\state) = 0$. The general result follows by induction. Assume that for $h+1$ the statement holds. Then:
\begin{align*}
\Valfun{\MDP\cup\reward}{\Policy}{\timestep}(\state)
&\overset{(a)}{=} \sum_{\action} \Policy_\timestep(\action|\state) \left( \reward_\timestep(\state, \action) + \sum_{\state'} \TransitionModel(\state'|\state, \action) \Valfun{\MDP\cup\reward}{\Policy}{\timestep+1}(\state') \right) \\
&\overset{(b)}{=} \sum_{\action} \Policy_\timestep(\action|\state) \left( \reward_\timestep(\state, \action) + \sum_{\state'} \TransitionModel(\state'|\state, \action) \left( \sum_{\timestep'=\timestep+1}^{\Horizon} \sum_{\state'',\action''} \Occupancy{\MDP}{\Policy}{\timestep+1}{\timestep'}(\state'',\action''|\state') \reward_{\timestep'}(\state'',\action'') \right) \right) \\
&\overset{(c)}{=} \sum_{\action} \Policy_\timestep(\action|\state) \reward_\timestep(\state, \action) + \sum_{\timestep'=\timestep+1}^{\Horizon} \sum_{\state',\action'} \Occupancy{\MDP}{\Policy}{\timestep}{\timestep'}(\state'|\state) \Policy_{\timestep'}(\action'|\state') \reward_{\timestep'}(\state',\action') \\
&\overset{(d)}{=} \sum_{\timestep'=\timestep}^{\Horizon} \sum_{\state',\action'} \Occupancy{\MDP}{\Policy}{\timestep}{\timestep'}(\state'|\state) \Policy_{\timestep'}(\action'|\state') \reward_{\timestep'}(\state',\action')
\end{align*}
where (a) uses the Bellman equation, (b) the induction step, (c) uses \Cref{def:occupancy} and relabels $\state''\to\state'$, $\action''\to\action'$, and (d) uses \Cref{def:occupancy} again and relabels $\action\to\action'$.
\end{proof}

\begin{lemma}[Simulation lemma 1 by \citet{metelli2021provably}]\label{simulation_lemma_1}
Let $\MDP$ be an \MDPnoR, and $\reward, \hat{\reward}$ two reward functions with corresponding optimal policies $\Policy^*, \hat{\Policy}^*$. Then,
\begin{align*}
\Qfun{\MDP\cup\reward}{\Policy^*}{\timestep}(\state, \action) - \Qfun{\MDP\cup\hat{\reward}}{\hat{\Policy}^*}{\timestep}(\state, \action)
\leq \sum_{\timestep'=\timestep}^{\Horizon} \sum_{\state', \action'} \Occupancy{\MDP}{\Policy^*}{\timestep}{\timestep'}(\state', \action' | \state, \action) \left( \reward_{\timestep'}(\state', \action') - \hat{\reward}_{\timestep'}(\state', \action') \right) \\
\Valfun{\MDP\cup\reward}{\Policy^*}{\timestep}(\state) - \Valfun{\MDP\cup\hat{\reward}}{\hat{\Policy}^*}{\timestep}(\state)
\leq \sum_{\timestep'=\timestep}^{\Horizon} \sum_{\state', \action'} \Occupancy{\MDP}{\Policy^*}{\timestep}{\timestep'}(\state', \action' | \state) \left( \reward_{\timestep'}(\state', \action') - \hat{\reward}_{\timestep'}(\state', \action') \right)
\end{align*}
\end{lemma}

\begin{proof}
Note that $\Qfun{\MDP\cup\hat{\reward}}{\hat{\Policy}^*}{\timestep}(\state, \action) \geq \Qfun{\MDP\cup\hat{\reward}}{\Policy^*}{\timestep}(\state, \action)$ for all $\state, \action$ because $\hat{\Policy}^*$ is optimal for $\hat{\reward}$. Hence
\begin{align*}
\Qfun{\MDP\cup\reward}{\Policy^*}{\timestep}(\state, \action) - \Qfun{\MDP\cup\hat{\reward}}{\hat{\Policy}^*}{\timestep}(\state, \action)
&\leq \Qfun{\MDP\cup\reward}{\Policy^*}{\timestep}(\state, \action) - \Qfun{\MDP\cup\hat{\reward}}{\Policy^*}{\timestep}(\state, \action) \\
&\overset{(a)}{=} \sum_{\timestep'=\timestep}^{\Horizon} \sum_{\state', \action'} \Occupancy{\MDP}{\Policy^*}{\timestep}{\timestep'}(\state',\action' | \state, \action) (\reward_{\timestep'}(\state', \action') - \hat{\reward}_{\timestep'}(\state', \action'))
\end{align*}
where (a) uses \Cref{qfun_occupancy_lemma}.
After observing $\Valfun{\MDP\cup\hat{\reward}}{\hat{\Policy}^*}{\timestep}(\state) \geq \Valfun{\MDP\cup\hat{\reward}}{\Policy^*}{\timestep}(\state)$, the second result follows analogously.
\end{proof}

\SimLemError*

\begin{proof}
\begin{align*}
&\Qfun{\MDP\cup\reward}{\Policy^*}{\timestep}(\state, \action) - \Qfun{\MDP\cup\reward}{\hat{\Policy}^*}{\timestep}(\state, \action)
= (\Qfun{\MDP\cup\reward}{\Policy^*}{\timestep}(\state, \action) - \Qfun{\MDP\cup\hat{\reward}}{\hat{\Policy}^*}{\timestep}(\state, \action)) + (\Qfun{\MDP\cup\hat{\reward}}{\hat{\Policy}^*}{\timestep}(\state, \action) - \Qfun{\MDP\cup\reward}{\hat{\Policy}^*}{\timestep}(\state, \action)) \\
&\overset{(a)}{\leq} \sum_{\timestep'=\timestep}^{\Horizon} \sum_{\state', \action'} \Occupancy{\MDP}{\Policy^*}{\timestep}{\timestep'}(\state', \action' | \state, \action) \left( \reward_{\timestep'}(\state', \action') - \hat{\reward}_{\timestep'}(\state', \action') \right) + (\Qfun{\MDP\cup\hat{\reward}}{\hat{\Policy}^*}{\timestep}(\state, \action) - \Qfun{\MDP\cup\reward}{\hat{\Policy}^*}{\timestep}(\state, \action)) \\
&\overset{(b)}{=} \sum_{\timestep'=\timestep}^{\Horizon} \sum_{\state', \action'}  \Occupancy{\MDP}{\Policy^*}{\timestep}{\timestep'}(\state', \action' | \state, \action) \left( \reward_{\timestep'}(\state', \action') - \hat{\reward}_{\timestep'}(\state', \action') \right) + \sum_{\timestep'=\timestep}^{\Horizon} \sum_{\state', \action'} \Occupancy{\MDP}{\hat{\Policy}^*}{\timestep}{\timestep'}(\state', \action' | \state, \action) \left( \hat{\reward}_{\timestep'}(\state', \action') - \reward_{\timestep'}(\state', \action') \right) \\
&= \sum_{\timestep'=\timestep}^{\Horizon} \sum_{\state', \action'} \left( \Occupancy{\MDP}{\Policy^*}{\timestep}{\timestep'}(\state', \action' | \state, \action) - \Occupancy{\MDP}{\hat{\Policy}^*}{\timestep}{\timestep'}(\state', \action' | \state, \action) \right) \left( \reward_{\timestep'}(\state', \action') - \hat{\reward}_{\timestep'}(\state', \action') \right)
\end{align*}
where (a) uses \Cref{simulation_lemma_1} and (b) uses \Cref{qfun_occupancy_lemma}.
\end{proof}

\begin{lemma}\label{dynamics_simulation_lemma_1}
Let $\MDP_1, \MDP_2$ be two \MDPnoR with transition dynamics $\TransitionModel_1$, $\TransitionModel_2$ respectively, $\reward$ a reward function and $\Policy$ a policy. Then, for any state $\state$ and timestep $\timestep$:

\resizebox{\linewidth}{!}{
  \begin{minipage}{\linewidth}
\begin{align*}
\Valfun{\MDP_2\cup\reward}{\Policy}{\timestep}(\state) - \Valfun{\MDP_1\cup\reward}{\Policy}{\timestep}(\state)
&= \sum_{\timestep'=\timestep}^\Horizon \sum_{\state',\action',\state''} \Occupancy{\MDP_2}{\Policy}{\timestep}{\timestep'}(\state';\state) \Policy_{\timestep'}(\action' | \state') ( \TransitionModel_2(\state'' | \state', \action') - \TransitionModel_1(\state'' | \state', \action') ) \Valfun{\MDP_1\cup\reward}{\Policy}{\timestep'+1}(\state'') \\
\Valfun{\MDP_1\cup\reward}{\Policy}{\timestep}(\state) - \Valfun{\MDP_2\cup\reward}{\Policy}{\timestep}(\state)
&= \sum_{\timestep'=\timestep}^\Horizon \sum_{\state',\action',\state''} \Occupancy{\MDP_2}{\Policy}{\timestep}{\timestep'}(\state';\state) \Policy_{\timestep'}(\action' | \state') ( \TransitionModel_1(\state'' | \state', \action') - \TransitionModel_2(\state'' | \state', \action') ) \Valfun{\MDP_1\cup\reward}{\Policy}{\timestep'+1}(\state'')
\end{align*}
\end{minipage}}
Moreover,

\resizebox{\linewidth}{!}{
  \begin{minipage}{\linewidth}
\begin{align*}
\Bigl| \Valfun{\MDP_2\cup\reward}{\Policy}{\timestep}(\state) - \Valfun{\MDP_1\cup\reward}{\Policy}{\timestep}(\state) \Bigr|
&\leq \sum_{\timestep'=\timestep}^\Horizon \sum_{\state',\action',\state''} \Occupancy{\MDP_2}{\Policy}{\timestep}{\timestep'}(\state';\state) \Policy_{\timestep'}(\action' | \state') \Bigl| \TransitionModel_2(\state'' | \state', \action') - \TransitionModel_1(\state'' | \state', \action') \Bigr| \Valfun{\MDP_1\cup\reward}{\Policy}{\timestep'+1}(\state'')
\end{align*}
\end{minipage}
}
\end{lemma}

\begin{proof}
We start by writing explicitly the value-functions:

\resizebox{\linewidth}{!}{
  \begin{minipage}{\linewidth}
\begin{align*}
&\Valfun{\MDP_2\cup\reward}{\Policy}{\timestep}(\state) - \Valfun{\MDP_1\cup\reward}{\Policy}{\timestep}(\state)
= \sum_{\action,\state'} \Policy_\timestep(\action|\state) \left( \TransitionModel_2(\state'|\state,\action) \Valfun{\MDP_2\cup\reward}{\Policy}{\timestep+1}(\state') - \TransitionModel_1(\state'|\state,\action) \Valfun{\MDP_1\cup\reward}{\Policy}{\timestep+1}(\state') \pm \TransitionModel_2(\state'|\state,\action) \Valfun{\MDP_1\cup\reward}{\Policy}{\timestep+1}(\state') \right) \\
=& \sum_{\action,\state'} \Policy_\timestep(\action|\state) \left( (\TransitionModel_2(\state'|\state,\action) - \TransitionModel_1(\state'|\state,\action)) \Valfun{\MDP_1\cup\reward}{\Policy}{\timestep+1}(\state') + \TransitionModel_2(\state'|\state,\action) (\Valfun{\MDP_2\cup\reward}{\Policy}{\timestep+1}(\state') - \Valfun{\MDP_1\cup\reward}{\Policy}{\timestep+1}(\state')) \right)
\end{align*}
\end{minipage}}

Unrolling the recursion gives the first result; the second result follows similarly:
\resizebox{\linewidth}{!}{
  \begin{minipage}{\linewidth}
  \begin{align*}
&\Valfun{\MDP_1\cup\reward}{\Policy}{\timestep}(\state) - \Valfun{\MDP_2\cup\reward}{\Policy}{\timestep}(\state)
= \sum_{\action,\state'} \Policy_\timestep(\action|\state) \left( \TransitionModel_1(\state'|\state,\action) \Valfun{\MDP_1\cup\reward}{\Policy}{\timestep+1}(\state') - \TransitionModel_2(\state'|\state,\action) \Valfun{\MDP_2\cup\reward}{\Policy}{\timestep+1}(\state') \pm \TransitionModel_2(\state'|\state,\action) \Valfun{\MDP_1\cup\reward}{\Policy}{\timestep+1}(\state') \right) \\
=& \sum_{\action,\state'} \Policy_\timestep(\action|\state) \left( (\TransitionModel_1(\state'|\state,\action) - \TransitionModel_2(\state'|\state,\action)) \Valfun{\MDP_1\cup\reward}{\Policy}{\timestep+1}(\state') + \TransitionModel_2(\state'|\state,\action) (\Valfun{\MDP_1\cup\reward}{\Policy}{\timestep+1}(\state') - \Valfun{\MDP_2\cup\reward}{\Policy}{\timestep+1}(\state')) \right)
\end{align*}
\end{minipage}}

Together, the first two results imply the third one because all terms in the sums are non-negative.
\end{proof}

\begin{lemma}\label{dynamics_simulation_lemma_2}
Let $\MDP_1, \MDP_2$ be two \MDPnoR with transition dynamics $\TransitionModel_1$, $\TransitionModel_2$ respectively, $\reward$ a reward function, and $\Policy_1^*, \Policy_2^*$ optimal policy in $\MDP_1\cup\reward$ and $\MDP_2\cup\reward$, respectively. Then, for any state $\state$ and timestep $\timestep$:
\begin{align*}
&\Valfun{\MDP_1\cup\reward}{*}{\timestep}(\state) - \Valfun{\MDP_2\cup\reward}{*}{\timestep}(\state)
\leq \sum_{\timestep'=\timestep} \sum_{\state',\action',\state''} \Occupancy{\MDP_2}{\Policy_1^*}{\timestep}{\timestep'}(\state';\state) \Policy_{1,\timestep}^*(\action'|\state') (\TransitionModel_1(\state''|\state',\action') - \TransitionModel_2(\state''|\state',\action')) \Valfun{\MDP_1\cup\reward}{*}{\timestep}(\state'') \\
&\Valfun{\MDP_2\cup\reward}{*}{\timestep}(\state) - \Valfun{\MDP_1\cup\reward}{*}{\timestep}(\state)
\leq \sum_{\timestep'=\timestep} \sum_{\state',\action',\state''} \Occupancy{\MDP_2}{\Policy_2^*}{\timestep}{\timestep'}(\state';\state) \Policy_{2,\timestep}^*(\action'|\state') (\TransitionModel_2(\state''|\state',\action') - \TransitionModel_1(\state''|\state',\action')) \Valfun{\MDP_2\cup\reward}{*}{\timestep}(\state'')
\end{align*}
\end{lemma}

\begin{proof}
\begin{align*}
\Valfun{\MDP_1\cup\reward}{*}{\timestep}(\state) - \Valfun{\MDP_2\cup\reward}{*}{\timestep}(\state)
= \sum_{\action,\state'} \Bigl( &\Policy_{1,\timestep}^*(\action|\state) \TransitionModel_1(\state'|\state,\action) \Valfun{\MDP_1\cup\reward}{\Policy_1^*}{\timestep+1}(\state') - \Policy_{2,\timestep}^*(\action|\state) \TransitionModel_2(\state'|\state,\action) \Valfun{\MDP_2\cup\reward}{\Policy_2^*}{\timestep+1}(\state') \\
&\pm \Policy_{1,\timestep}^*(\action|\state) \TransitionModel_2(\state'|\state,\action) \Valfun{\MDP_1\cup\reward}{\Policy_1^*}{\timestep+1}(\state') \pm \Policy_{1,\timestep}^*(\action|\state) \TransitionModel_2(\state'|\state,\action) \Valfun{\MDP_2\cup\reward}{\Policy_2^*}{\timestep+1}(\state') \Bigr) \\
= \sum_{\action,\state'} \Bigl( &\Policy_{1,\timestep}^*(\action|\state) \TransitionModel_2(\state'|\state,\action) (\Valfun{\MDP_1\cup\reward}{\Policy_1^*}{\timestep+1}(\state') - \Valfun{\MDP_2\cup\reward}{\Policy_2^*}{\timestep+1}(\state')) \\
&+ \Policy_{1,\timestep}^*(\action|\state) (\TransitionModel_1(\state'|\state,\action) - \TransitionModel_2(\state'|\state,\action)) \Valfun{\MDP_1\cup\reward}{\Policy_1^*}{\timestep+1}(\state') \\
&+ (\Policy_{1,\timestep}^*(\action|\state) - \Policy_{2,\timestep}^*(\action|\state)) \TransitionModel_2(\state'|\state,\action) \Valfun{\MDP_2\cup\reward}{\Policy_2^*}{\timestep+1}(\state') \Bigr) \\
\leq \sum_{\action,\state'} \Bigl( &\Policy_{1,\timestep}^*(\action|\state) \TransitionModel_2(\state'|\state,\action) (\Valfun{\MDP_1\cup\reward}{\Policy_1^*}{\timestep+1}(\state') - \Valfun{\MDP_2\cup\reward}{\Policy_2^*}{\timestep+1}(\state')) \\
&+ \Policy_{1,\timestep}^*(\action|\state) (\TransitionModel_1(\state'|\state,\action) - \TransitionModel_2(\state'|\state,\action)) \Valfun{\MDP_1\cup\reward}{\Policy_1^*}{\timestep+1}(\state') \Bigr)
\end{align*}
where the last inequality uses that $\Policy^*$ is optimal for $\MDP_2 \cup\reward$. Unrolling the recursion gives the first result. A similar argument yields the second results:
\begin{align*}
\Valfun{\MDP_2\cup\reward}{*}{\timestep}(\state) - \Valfun{\MDP_1\cup\reward}{*}{\timestep}(\state)
= \sum_{\action,\state'} \Bigl( &\Policy_{2,\timestep}^*(\action|\state) \TransitionModel_2(\state'|\state,\action) \Valfun{\MDP_2\cup\reward}{\Policy_2^*}{\timestep+1}(\state') - \Policy_{1,\timestep}^*(\action|\state) \TransitionModel_1(\state'|\state,\action) \Valfun{\MDP_1\cup\reward}{\Policy_1^*}{\timestep+1}(\state') \\
&\pm \Policy_{2,\timestep}^*(\action|\state) \TransitionModel_2(\state'|\state,\action) \Valfun{\MDP_1\cup\reward}{\Policy_1^*}{\timestep+1}(\state') \Bigr) \\
= \sum_{\action,\state'} \Bigl( &\Policy_{2,\timestep}^*(\action|\state) \TransitionModel_2(\state'|\state,\action) (\Valfun{\MDP_2\cup\reward}{\Policy_2^*}{\timestep+1}(\state') - \Valfun{\MDP_1\cup\reward}{\Policy_1^*}{\timestep+1}(\state')) \\
&+ \Policy_{2,\timestep}^*(\action|\state) \TransitionModel_2(\state'|\state,\action) \Valfun{\MDP_1\cup\reward}{\Policy_1^*}{\timestep+1}(\state') - \Policy_{1,\timestep}^*(\action|\state) \TransitionModel_1(\state'|\state,\action) \Valfun{\MDP_1\cup\reward}{\Policy_1^*}{\timestep+1}(\state') \\
\leq \sum_{\action,\state'} \Bigl( &\Policy_{2,\timestep}^*(\action|\state) \TransitionModel_2(\state'|\state,\action) (\Valfun{\MDP_2\cup\reward}{\Policy_2^*}{\timestep+1}(\state') - \Valfun{\MDP_1\cup\reward}{\Policy_1^*}{\timestep+1}(\state')) \\
&+ \Policy_{2,\timestep}^*(\action|\state) ( \TransitionModel_2(\state'|\state,\action) - \TransitionModel_1(\state'|\state,\action) ) \Valfun{\MDP_1\cup\reward}{\Policy_1^*}{\timestep+1}(\state')
\end{align*}
\end{proof}

\subsection{Feasible Reward Set}

In this section, we characterize the feasible reward set first implicitly, then explicitly, and prove a result about error propagation. \cite{metelli2021provably} provide a similar analysis in the infinite horizon setting.

\LemmaIfh*

\begin{proof}
The result follows directly from \Cref{def:feasible_reward_set}.
\end{proof}

\begin{lemma}\label{lemma:q_expert_explicit}
A Q-function satisfies the conditions of \Cref{lemma:ifh} if and only if there exists an $\{A_h \in \reals_{\ge 0}^{\StateSpace\times\ActionSpace}\}_{\timestep \in \Horizon}$ and $\{V_h \in \reals^{\StateSpaceSize}\}$ such that for every $\timestep, \state, \action \in [\Horizon] \times\StateSpace\times\ActionSpace$:
\begin{align*}
\Qfun{\MDP\cup\reward}{\ExpertPolicy}{\timestep}(\state,\action) = -A_\timestep(\state,\action)\indicator{\ExpertPolicy_\timestep(\action|\state) = 0} + V_\timestep(\state)
\end{align*}
\end{lemma}

\begin{proof}
We first show that if $\Qfun{\MDP\cup\reward}{\ExpertPolicy}{\timestep}(\state,\action)$ has this form, the conditions of \Cref{lemma:ifh} are satisfied, and then the converse.
Assume $\Qfun{\MDP\cup\reward}{\ExpertPolicy}{\timestep}(\state,\action) = -A_\timestep(\state,\action)\indicator{\ExpertPolicy_\timestep(\action|\state) = 0} + V_\timestep(\state)$. Then,
\[
\Valfun{\MDP\cup\reward}{\ExpertPolicy}{\timestep}(\state) = \sum_{\action} \ExpertPolicy_\timestep(\action|\state) \Qfun{\MDP\cup\reward}{\ExpertPolicy}{\timestep}(\state,\action) = V_\timestep(\state).
\]
If $\ExpertPolicy_\timestep(\action|\state) > 0$, then $\Qfun{\MDP\cup\reward}{\ExpertPolicy}{\timestep}(\state,\action) = \Valfun{\MDP\cup\reward}{\ExpertPolicy}{\timestep}(\state)$, which is the first condition of \Cref{lemma:ifh}. If $\ExpertPolicy_\timestep(\action|\state) = 0$, $\Qfun{\MDP\cup\reward}{\ExpertPolicy}{\timestep}(\state,\action) = \Valfun{\MDP\cup\reward}{\ExpertPolicy}{\timestep}(\state) - A_\timestep(\state, \action) \leq \Valfun{\MDP\cup\reward}{\ExpertPolicy}{\timestep}(\state)$, which is the second condition of \Cref{lemma:ifh}.

For the converse, assume that the conditions of \Cref{lemma:ifh} hold, and let $V_\timestep(\state) = \Valfun{\MDP\cup\reward}{\ExpertPolicy}{\timestep}(\state)$ and $A_\timestep(\state, \action) = \Valfun{\MDP\cup\reward}{\ExpertPolicy}{\timestep}(\state) - \Qfun{\MDP\cup\reward}{\ExpertPolicy}{\timestep}(\state,\action)$.
\end{proof}

\LemmaEfh*

\begin{proof}
Since $\Qfun{\MDP\cup\reward}{\ExpertPolicy}{\timestep}(\state,\action) = \reward_\timestep(\state,\action) + \sum_{\state'} \TransitionModel(\state'|\state,\action)V_{\timestep+1}(\state')$, using \Cref{lemma:q_expert_explicit}, we have:
\begin{align*}
    \reward_h(\state,\action) &= \Qfun{\MDP\cup\reward}{\ExpertPolicy}{\timestep}(\state,\action) - \sum_{\state'} \TransitionModel(\state'|\state,\action)V_{\timestep+1}(\state') \\
    &= -A_\timestep(\state,\action)\indicator{\ExpertPolicy_\timestep(\action|\state) = 0} + V_\timestep(\state) + \sum_{\state'} \TransitionModel(\state'|\state,\action){V}_{\timestep+1}(\state')
\end{align*}
\end{proof}

\ThmError*

\begin{proof}
We start by rewriting $\reward$ and $\hat{\reward}$ using \Cref{lemma:efh}:
\begin{align*}
\reward_{\timestep}(\state,\action) &= -A_{\timestep}(\state,\action)\indicator{\ExpertPolicy_{\timestep}(\action|\state) = 0} + V_{\timestep}(\state) + \sum_{\state'} \TransitionModel(\state'|\state,\action)V_{\timestep+1}(\state') \\
\hat{\reward}_{\timestep}(\state,\action) &= -\hat{A}_{\timestep}(\state,\action)\indicator{\hat{\Policy}^E_{\timestep}(\action|\state) = 0} + \hat{V}_{\timestep}(\state) + \sum_{\state'} \EstTransitionModel(\state'|\state,\action)\hat{V}_{\timestep+1}(\state')
\end{align*}
We can choose (w.l.o.g.) $V_{\timestep} = \hat{V}_{\timestep}$ and $\hat{A}_{\timestep} = A_{\timestep}$:
\begin{align*}
\reward_{\timestep}(\state,\action) - \hat{\reward}_{\timestep}(\state,\action) =& -A_{\timestep}(\state,\action)\indicator{\ExpertPolicy_{\timestep}(\action|\state) = 0} + V_\timestep(\state) + \sum_{\state'} \TransitionModel(\state'|\state,\action)V_{\timestep+1}(\state') \\
&+A_{\timestep}(\state,\action)\indicator{\hat{\Policy}^E_{\timestep}(\action|\state)=0} - V_\timestep(\state) - \sum_{\state'} \EstTransitionModel(\state'|\state,\action)V_{\timestep+1}(\state') \\
=& -A_\timestep(\state,\action) (\indicator{\ExpertPolicy_{\timestep}(\action|\state) = 0} - \indicator{\hat{\Policy}^E_{\timestep}(\action|\state)=0})
+ \sum_{\state'} V_{\timestep+1}(\state') (\TransitionModel(\state'|\state,\action)-\EstTransitionModel(\state'|\state,\action))
\end{align*}
Note that $\indicator{\ExpertPolicy_{\timestep}(\action|\state) = 0} \leq 1 - \ExpertPolicy_{\timestep}(\action|\state)$ and $\indicator{\hat{\Policy}^E_{\timestep}(\action|\state)=0} \leq 1 - \hat{\Policy}^E_{\timestep}(\action|\state)$, which can be easily checked for both cases of the indicator. Using this, we get
\begin{align*}
\reward_{\timestep}(\state,\action) - \hat{\reward}_{\timestep}(\state,\action)
\leq A_\timestep(\state,\action) (\ExpertPolicy_{\timestep}(\action|\state) - \hat{\Policy}^E_{\timestep}(\action|\state))
+ \sum_{\state'} V_{\timestep+1}(\state') (\TransitionModel(\state'|\state,\action)-\EstTransitionModel(\state'|\state,\action))
\end{align*}
The result follows by taking the absolute value and applying the triangle inequality.
\end{proof}

\subsection{Uniform Sampling IRL with a Generative Model}
\label{app:uniform_sampling}

In this section, we derive sample complexity results for uniform sampling with a generative model (\Cref{alg:uniform_irl_generative}). \citet{metelli2021provably} proved an analogous result for the infinite horizon setting focusing on transferable rewards. In contrast, our focus is on the finite horizon setting. Moreover,  \cite{metelli2021provably} considers to learn a reward that is transferable to a known target environment. In our setting, instead, we suppose to use the recovered reward function in the unknown source environment.

\begin{algorithm}[t]
\caption{Uniform sampling IRL with a generative model.}
\label{alg:uniform_irl_generative}
\begin{algorithmic}[1]
    \State \textbf{Input:} significance $\delta \in (0,1)$, target accuracy $\epsilon$, maximum number of samples per iter. $n_\mathrm{max}$
    \State Initialize $\episode \gets 0$, $\epsilon_0 \gets \Horizon$
    \While{$\epsilon_\episode > \epsilon / 2$}
        \State Uniformly sample $\lceil \frac{n_\mathrm{max}}{\StateSpaceSize \ActionSpaceSize \Horizon} \rceil$ samples from all $(\state, \action, \timestep) \in \StateSpace \times \ActionSpace \times [\Horizon]$
        \State For all samples, observe sample from transition dynamics and expert policy
        \State $\episode \gets \episode + 1$
        \State Update $\EstTransitionModel_\episode$, $\hat{\Policy}_\episode$, and $\rewardci_\episode^\timestep$
        \State Update accuracy $\epsilon_\episode \gets \Horizon \max_{\state, \action, \timestep} \rewardci^\timestep_\episode(\state, \action)$
    \EndWhile
\end{algorithmic}
\end{algorithm}

\OptimalityCriterion*

\begin{lemma}[Good Event]\label{lemma:good_event}
Let $\Policy^E$ be a (possibly stochastic) expert policy. We estimate the expert policy with $\EstPol^E$ and the transition model $\TransitionModel$ with an estimate $\EstTransitionModel_\episode$ from $\episode$ episodic interactions. Let $n^\timestep_\episode(\state,\action)$ and $n^\timestep_\episode(\state)$ be the number of times state action pairs and states have been observed at time $\timestep$ within the first $\episode$ episodes, and $n^{\timestep+}_\episode(\state, \action) = \max\{1, n^\timestep_\episode(\state,\action)\}$. Then,
\begin{align*}
| \ExpertPolicy_{\timestep}(\action|\state) - \hat{\Policy}^E_{\timestep}(\action|\state) A_\timestep(\state,\action) | &\leq (\Horizon-\timestep) \Rmax \sqrt{\frac{\ell^\timestep_\episode(\state, \action)}{{n^\timestep_\episode}^+(\state, \action)}} \\
| \ExpertPolicy_{\timestep}(\action|\state) - \hat{\Policy}^E_{\timestep}(\action|\state) \hat{A}_\timestep(\state,\action) | &\leq (\Horizon-\timestep) \Rmax \sqrt{\frac{\ell^\timestep_\episode(\state, \action)}{{n^\timestep_\episode}^+(\state, \action)}} \\
\sum_{\state'} | (\TransitionModel(\state' | \state, \action) - \hat{\TransitionModel}_\episode( \state' | \state, \action)) \Valfun{\reward}{\pi}{\timestep}(\state') | &\leq (\Horizon-\timestep) \Rmax \sqrt{\frac{2\ell^\timestep_\episode(\state, \action)}{{n^\timestep_\episode}^+(\state, \action)}} \\
\sum_{\state'} | (\TransitionModel(\state' | \state, \action) - \hat{\TransitionModel}_\episode(\state' | \state, \action)) \EstValfun{\reward}{\pi}{\timestep}(\state') | &\leq (\Horizon-\timestep) \Rmax \sqrt{\frac{2\ell^\timestep_\episode(\state, \action)}{{n^\timestep_\episode}^+(\state, \action)}}
\end{align*}
where $\ell^\timestep_\episode(\state, \action) = \log\left( 24 \StateSpaceSize \ActionSpaceSize \Horizon ({n^\timestep_\episode}^+(\state, \action))^2 / \delta \right)$, holds simultaneously for all $(\state, \action, \timestep) \in \StateSpace \times \ActionSpace \times [\Horizon]$ and $\episode \geq 1$ with probability at least $1-\delta$. We call the event that these equations hold the \emph{good event} $\GoodEvent$ and write $P(\GoodEvent) \geq 1-\delta$.
\end{lemma}

\begin{proof}
We show that each statement individually does not hold with probability less than $\delta/4$, which implies the result via a union bound. Let us denote $\beta_1(\state, \action, \timestep) \coloneqq (\Horizon-\timestep) \Rmax \sqrt{\frac{2\ell^\timestep_\episode(\state, \action)}{{n^\timestep_\episode}^+(\state, \action)}}$. First, consider the last two inequalities. The probability that either of them does not hold is:

\begin{align*}
&\Pr\left( \exists \episode \geq 1, (\state, \action, \timestep) \in \StateSpace\times\ActionSpace\times[\Horizon] : \sum_{\state'} | (\TransitionModel(\state' | \state, \action) - \EstTransitionModel_\episode( \state' | \state, \action)) \Valfun{\reward}{\Policy}{\timestep}(\state') | > \beta_1(\state, \action, \timestep) \right) \\
\overset{(a)}{\leq} & \Pr\left( \exists m \geq 0, (\state, \action, \timestep) \in \StateSpace\times\ActionSpace\times[\Horizon] : \sum_{\state'} | (\TransitionModel(\state' | \state, \action) - \EstTransitionModel_\episode( \state' | \state, \action)) \Valfun{\reward}{\Policy}{\timestep}(\state') | > \beta_1(\state, \action, \timestep) \right)\\
\overset{(b)}{\leq} & \sum_{m \geq 0} \sum_{\state,\action} \sum_{\timestep=0}^\Horizon \Pr\left(\sum_{\state'} | (\TransitionModel(\state' | \state, \action) - \EstTransitionModel_\episode( \state' | \state, \action)) \Valfun{\reward}{\Policy}{\timestep}(\state') | > \beta_1(\state, \action, \timestep) \right) \\
\overset{(c)}{\leq} & \sum_{m \geq 0} \sum_{\state,\action} \sum_{\timestep = 0}^{\Horizon} 2 \exp\left( - \frac{2 \beta_1(\state, \action, \timestep)^2 m^2}{4 m (\Horizon-\timestep)^2 \Rmax^2} \right)
\leq \sum_{m \geq 0} \sum_{\state,\action} \sum_{\timestep = 0}^{\Horizon} 2 \exp\left( - \ell_\episode(\state, \action) \right) \\
= &\sum_{m \geq 0} \sum_{\state,\action} \sum_{\timestep = 0}^{\Horizon} \frac{2 \delta}{24 \StateSpaceSize \ActionSpaceSize \Horizon (m^+)^2}
= \frac{\delta}{12} \left( 1 + \sum_{m\geq0} \frac{1}{m^2} \right)
= \frac{\delta}{12} \left( 1 + \frac{\pi^2}{6} \right)
\leq \frac{\delta}{4}
\end{align*}
Step (a) assumes that we visit a state action pair $m$ times, and focuses on these $m$ times the transition model for the given state-action pair is updated. Step (b) uses a union bound over $m$ and $(\state, \action)$. Step (c) applies Hoeffding's inequality using that we estimate $\TransitionModel$ with an average of samples, and $\Valfun{\reward}{\Policy}{\timestep} \leq (\Horizon-\timestep) \Rmax$. The factor $m^2$ in the numerator results from dividing by $1/m$ to average over samples, and the factor $4m$ in the denominator results from the sum over $m$ in the denominator of Hoeffding's bound.

We show the first two inequalities similarly, with $\beta_2(\state, \action, \timestep) \coloneqq (\Horizon-\timestep) \Rmax \sqrt{\frac{\ell^\timestep_\episode(\state, \action)}{{n^\timestep_\episode}^+(\state, \action)}}$

\begin{align*}
    &\Pr\left( \exists \episode \geq 1, (\state, \action, \timestep) \in \StateSpace\times\ActionSpace\times[\Horizon] : | (\pi^E_\episode(\action | \state) - \hat{\pi}^E_\episode(\action | \state)) A_\timestep(\state,\action) | > \beta_2(\state, \action, \timestep) \right) \\
    \overset{(a)}{\leq} & \Pr\left( \exists m \geq 0, (\state, \action, \timestep) \in \StateSpace\times\ActionSpace\times[\Horizon] :| (\pi^E_\episode(\action | \state) - \hat{\pi}^E_\episode(\action | \state))A_\timestep(\state,\action) | > \beta_2(\state, \action, \timestep) \right)\\
\overset{(b)}{\leq} & \sum_{m \geq 0} \sum_{\state,\action} \sum_{\timestep=0}^\Horizon \Pr\left(| (\pi^E_\episode(\action | \state) - \hat{\pi}^E_\episode(\action | \state)) A_\timestep(\state,\action) | > \beta_2(\state, \action, \timestep) \right) \\
\overset{(c)}{\leq} & \sum_{m \geq 0} \sum_{\state,\action} \sum_{\timestep = 0}^{\Horizon} 2 \exp\left( - \frac{2 \beta_2(\state, \action, \timestep)^2m^2}{4 m (\Horizon-\timestep)^2 \Rmax^2} \right)
\leq \sum_{m \geq 0} \sum_{\state,\action} \sum_{\timestep = 0}^{\Horizon} 2 \exp\left( - \ell_\episode(\state, \action) \right) \\
= &\sum_{m \geq 0} \sum_{\state,\action} \sum_{\timestep = 0}^{\Horizon} \frac{2 \delta}{24 \StateSpaceSize \ActionSpaceSize \Horizon (m^+)^2}
= \frac{\delta}{12} \left( 1 + \sum_{m\geq0} \frac{1}{m^2} \right)
= \frac{\delta}{12} \left( 1 + \frac{\pi^2}{6} \right)
\leq \frac{\delta}{4}
\end{align*}
A union bound over all equations results in $P(\GoodEvent) \geq 1-\delta$.
\end{proof}

\begin{definition}\label{def:rewardci}
We define the reward uncertainty as
\[
\rewardci^\timestep_\episode(\state, \action) = (\Horizon-\timestep) \Rmax \min \left(1, 2 \sqrt{\frac{2 \ell^\timestep_\episode(\state, \action)}{{n^\timestep_\episode}(\state,\action)}} \right)
\]
\end{definition}

\begin{corollary}\label{cor:rewarci}
Under the good event $\GoodEvent$, in each iteration $\episode$ it holds for all $(\state, \action, \timestep) \in \StateSpace\times\ActionSpace\times[\Horizon]$ that:
\[
|\reward_\timestep(\state, \action) - \hat{\reward}_\timestep^\episode(\state,\action) | \leq \rewardci^\timestep_\episode(\state, \action)
\]
\end{corollary}

\begin{proof}
\begin{align*}
|\reward_\timestep(\state,\action) - \hat{\reward}_\timestep^\episode(\state,\action)| &\overset{(a)}{\le} A_\timestep(\state,\action) | \ExpertPolicy_\timestep(\action|\state) - \EstPol^E_\timestep(\action|\state)| + \sum_{\state'} V_{\timestep+1}(\state') |\TransitionModel(\state'|\state,\action)-\EstTransitionModel(\state'|\state,\action)| \\
&\overset{(b)}{\leq} (\Horizon-\timestep) \Rmax \left( 2 \sqrt{\frac{2 \ell^\timestep_\episode(\state, \action)}{{n^\timestep_\episode}^+(\state,\action)}} \right) = \rewardci^\timestep_\episode(\state, \action)
\end{align*}
where (a) uses \Cref{thm:error_propagation} and (b) uses \Cref{lemma:good_event}.
\end{proof}

\begin{corollary}\label{cor:uniform_sampling_stopping}
Let $\mathscr{S}$ be a sampling strategy. Let $\ExactFeasibleSet$ be the exact feasible set and $\RecoveredFeasibleSetIter$ be the feasible set recovered after $\episode$ iterations. If
\[
\Horizon \max_{\state,\action,\timestep} \rewardci^\timestep_\episode(\state,\action) \leq \frac{\epsilon}{2},
\]
then the conditions of \Cref{def:correct} are satisfied.
\end{corollary}

\begin{proof}
For the first condition of \Cref{def:correct}, observe:
\begin{align*}
&\inf_{\hat{\reward}\in\RecoveredFeasibleSetIter} \sup_{\hat{\Policy}^* \in \Pi^*_{\EstMDP\cup\hat{\reward}}} \max_{\state,\action,\timestep} ( \Qfun{\MDP \cup \reward}{\Policy^*}{\timestep}(\state,\action) - \Qfun{\MDP \cup \reward}{\hat{\Policy}^*}{\timestep}(\state,\action) ) \\ 
\overset{(a)}{\leq}& \inf_{\hat{\reward}\in\RecoveredFeasibleSetIter} \sup_{\hat{\Policy}^* \in \Pi^*_{\EstMDP\cup\hat{\reward}}} \max_{\state,\action,\timestep} \sum_{\timestep'=\timestep}^{\Horizon} \sum_{\state', \action'} \left(\Occupancy{\MDP}{\Policy^*}{\timestep}{\timestep'}(\state', \action' | \state, \action) - \Occupancy{\MDP}{\hat{\Policy}^*}{\timestep}{\timestep'}(\state', \action' | \state, \action)\right) \left( \reward_{\timestep'}(\state', \action') - \hat{\reward}_{\timestep'}(\state', \action') \right) \\
\overset{(b)}{\leq}& \inf_{\hat{\reward}\in\RecoveredFeasibleSetIter} \sup_{\hat{\Policy}^* \in \Pi^*_{\EstMDP\cup\hat{\reward}}} \max_{\state,\action,\timestep} \Bigl| \sum_{\timestep'=\timestep}^{\Horizon} \sum_{\state', \action'} \left(\Occupancy{\MDP}{\Policy^*}{\timestep}{\timestep'}(\state', \action' | \state, \action) - \Occupancy{\MDP}{\hat{\Policy}^*}{\timestep}{\timestep'}(\state', \action' | \state, \action)\right) \rewardci_\episode^{\timestep'}(\state',\action') \Bigr| \\
\leq& 2 \Horizon \max_{\state,\action,\timestep} \rewardci_\episode^{\timestep}(\state,\action)
\end{align*}
where (a) uses \Cref{simulation_lemma_same_reward_different_policies} and (b) uses \Cref{cor:rewarci}.

For the second condition of \Cref{def:correct}, it follows similarly that:
\begin{align*}
\inf_{\reward\in\ExactFeasibleSet} \sup_{\Policy^*\in\Pi^*_{\MDP\cup\reward}} \max_{\state,\action,\timestep} ( \Qfun{\MDP\cup\reward}{\Policy^*}{\timestep}(\state,\action) - \Qfun{\MDP\cup\reward}{\hat{\Policy}^*}{\timestep}(\state,\action) ) 
\leq 2 \Horizon \max_{\state,\action,\timestep} \rewardci_\episode^{\timestep}(\state,\action)
\end{align*}
Hence, if $\Horizon \max_{\state,\action,\timestep} \rewardci^\timestep_\episode(\state,\action) \leq \epsilon / 2$, both conditions of \Cref{def:correct} are satisfied.
\end{proof}

\begin{theorem}[Sample Complexity of Uniform Sampling IRL]
With probability at least $1-\delta$, \Cref{alg:uniform_irl_generative} stops at iteration $\tau$ fulfilling \Cref{def:correct} with a number of samples upper bounded by:
\[
n \leq \tilde{\mathcal{O}} \left( \frac{\Horizon^5 \Rmax^2 \StateSpaceSize \ActionSpaceSize}{\epsilon^2} \right)
\]
\end{theorem}

\begin{proof}
First, note
\[
\Horizon \max_{\state,\action,\timestep} \rewardci^\timestep_\episode(\state,\action)
= \Horizon^2 \Rmax \max_{\state,\action,\timestep} \left( 2 \sqrt{\frac{2 \ell^\timestep_k(\state, \action)}{{n^\timestep_\episode}^+(\state,\action)}} \right)
\]

After $\tau$ iterations, we have collected $\tau \cdot n_{\max}$ samples and for each $\state,\action,\timestep$, we have:
${n^\timestep_\tau}^+(\state,\action) \geq \frac{\tau n_{\max}}{\StateSpaceSize\ActionSpaceSize\Horizon} \geq 1$

To terminate at iteration $\tau$, we need to have for all $\state,\action,\timestep$:
\begin{align*}
2 \Horizon^2 \Rmax \sqrt{\frac{2 \ell^\timestep_\tau(\state,\action)}{n^\timestep_\tau(\state,\action)}} \leq \frac{\epsilon}{2}
\end{align*}
which implies
\begin{align*}
n^\timestep_\tau(\state,\action) \geq \frac{32\Horizon^4 \Rmax^2 \ell^\timestep_\tau(\state,\action)}{\epsilon^2}
\end{align*}

By using Lemma B.8 by \cite{metelli2021provably}, we can conclude that the number of samples necessary to ensure accuracy $\varepsilon$ is:
\[
n \leq \tilde{\mathcal{O}} \left( \frac{\Horizon^5 \Rmax^2 \StateSpaceSize \ActionSpaceSize}{\epsilon^2} \right)
\]
\end{proof}

\begin{corollary}
If the true reward function does not depend on the timestep $\timestep$, i.e., $\reward_\timestep(\state,\action)=\reward(\state,\action)$, then we can modify \Cref{alg:uniform_irl_generative} to only need $n \leq \tilde{\mathcal{O}} \left( \frac{\Horizon^4 \Rmax^2 \StateSpaceSize \ActionSpaceSize}{\epsilon^2} \right)$ samples.
\end{corollary}

\begin{proof}
If we know that the reward function does not depend on $\timestep$ we can choose $\rewardci_\episode(\state,\action) = \min_\timestep \rewardci^\timestep_\episode(\state,\action)$ as a confidence interval of the reward. Consequently, we can sample all states for a fixed $\timestep$.

We still need for all $\state,\action$:
\begin{align*}
2 \Horizon^2 \Rmax \sqrt{\frac{2 \ell^\timestep_\tau(\state,\action)}{n^\timestep_\tau(\state,\action)}} \leq \frac{\epsilon}{2}
~~\Rightarrow~~ n^\timestep_\tau(\state,\action) \geq \frac{32 \Horizon^4 \Rmax^2 \ell^\timestep_\tau(\state,\action)}{\epsilon^2}
\end{align*}

Again, we use Lemma B.8 by \cite{metelli2021provably}, but we can eliminate one sum over $\Horizon$, ending up with:
\[
n \leq \tilde{\mathcal{O}} \left( \frac{\Horizon^4 \Rmax^2 \StateSpaceSize \ActionSpaceSize}{\epsilon^2} \right)
\]
\end{proof}

\subsection{Sample Complexity of \AlgNameShort in Unknown Environments (Problem Independent)}
\label{app:sample-complexity}

We are now ready to analyze the sample complexity of \AlgNameShort (\Cref{alg:aceirl}). We first consider the simple version of the algorithm: \AlgNameShort Greedy. Then, we consider the full version of the algorithm after introducing a few additional lemma about the policy confidence set. 
We start by defining the error upper bound and deriving two lemmas that will help us to show that it is indeed an upper bound on the error we want to reduce.

\begin{definition}\label{def:ucrl_error}
We define recursively:
\begin{align*}
E_\episode^\Horizon(\state, \action) = 0; \quad
E_\episode^\timestep(\state, \action) = \min\Bigl((\Horizon-\timestep)\Rmax, \rewardci_\episode^\timestep(\state, \action) + \sum_{\state^\prime} \EstTransitionModel(\state^\prime | \state, \action) \max_{\action^\prime \in \ActionSpace}  E_\episode^{\timestep+1}(\state^\prime, \action^\prime) \Bigr)
\end{align*}
where $\EstTransitionModel$ is the estimated transition model of the environment.
\end{definition}

The first lemma shows that the error upper bound can upper bound the error due to estimating the transition model.

\begin{lemma}\label{lemma:ucrl_error_estmdp}
Under the good event $\GoodEvent$, for all policies $\Policy$ and reward functions $\reward$ and all $\state,\action,\timestep$:
\[
| \Qfun{\EstMDP\cup\reward}{\Policy}{\timestep}(\state, \action) - \Qfun{\MDP\cup\reward}{\Policy}{\timestep}(\state, \action) | \leq E_\episode^\timestep(\state, \action)
\]
\end{lemma}

\begin{proof}
\begin{align*}
&| \Qfun{\EstMDP\cup\reward}{\Policy}{\timestep}(\state,\action) - \Qfun{\MDP\cup\reward}{\Policy}{\timestep}(\state,\action) |
= \Bigl| \sum_{\state'} \EstTransitionModel(\state'|\state,\action) \sum_{\action'} \Policy(\action'|\state') \Qfun{\EstMDP\cup\reward}{\Policy}{\timestep+1}(\state',\action') \\
&- \sum_{\state'} \TransitionModel(\state'|\state,\action) \sum_{\action'} \Policy(\action'|\state') \Qfun{\MDP\cup\reward}{\Policy}{\timestep+1}(\state',\action')
\pm \sum_{\state'} \EstTransitionModel(\state'|\state,\action) \sum_{\action'} \Policy(\action'|\state') \Qfun{\MDP\cup\reward}{\Policy}{\timestep+1}(\state',\action')
\Bigr| \\
&\leq \Bigl| \sum_{\state'} \bigl( \EstTransitionModel(\state'|\state,\action) - \TransitionModel(\state'|\state,\action) \bigr) \sum_{\action'} \Policy(\action'|\state') \Qfun{\MDP\cup\reward}{\Policy}{\timestep+1}(\state',\action') \Bigr| \\
&+ \sum_{\state'} \EstTransitionModel(\state'|\state,\action) \sum_{\action'} \Policy(\action'|\state') \Bigl| \Qfun{\EstMDP\cup\reward}{\Policy}{\timestep+1}(\state,\action) - \Qfun{\MDP\cup\reward}{\Policy}{\timestep+1}(\state,\action) \Bigr| \\
&\leq \rewardci_\episode^\timestep(\state,\action) + \sum_{\state'} \EstTransitionModel(\state'|\state,\action) \sum_{\action'} \Policy(\action'|\state') \Bigl| \Qfun{\EstMDP\cup\reward}{\Policy}{\timestep+1}(\state,\action) - \Qfun{\MDP\cup\reward}{\Policy}{\timestep+1}(\state,\action) \Bigr|
\end{align*}

For $\timestep=\Horizon$ the result holds trivially. Now assuming it holds for $\timestep+1$, we consider step $\timestep$:
\begin{align*}
&| \Qfun{\EstMDP\cup\reward}{\Policy}{\timestep}(\state,\action) - \Qfun{\MDP\cup\reward}{\Policy}{\timestep}(\state,\action) |
\leq \rewardci_\episode^\timestep(\state,\action) + \sum_{\state'} \EstTransitionModel(\state'|\state,\action) \sum_{\action'} \Policy(\action'|\state') \Bigl| \Qfun{\EstMDP\cup\reward}{\Policy}{\timestep+1}(\state,\action) - \Qfun{\MDP\cup\reward}{\Policy}{\timestep+1}(\state,\action) \Bigr| \\
&\leq \rewardci_\episode^\timestep(\state,\action) + \sum_{\state'} \EstTransitionModel(\state'|\state,\action) \max_{\action'} \Bigl| \Qfun{\EstMDP\cup\reward}{\Policy}{\timestep+1}(\state,\action) - \Qfun{\MDP\cup\reward}{\Policy}{\timestep+1}(\state,\action) \Bigr| \\
&\leq \rewardci_\episode^\timestep(\state,\action) + \sum_{\state'} \EstTransitionModel(\state'|\state,\action) \max_{\action'} E_\episode^{\timestep+1}(\state', \action')
=  E_\episode^{\timestep}(\state, \action)
\end{align*}
\end{proof}

The next lemma shows that the error upper bound can also upper bound the error in estimating the reward function, which is due to estimating the transition model and the expert policy.

\begin{lemma}\label{lemma:ucrl_error_estreward}
Under the good event $\GoodEvent$, for all reward function $\reward$, all policies $\Policy$, and all $\state,\action \in \StateSpace \times \ActionSpace$:
\[
| \Qfun{\EstMDP\cup\hat{\reward}}{\Policy}{\timestep}(\state, \action) - \Qfun{\EstMDP\cup\reward}{\Policy}{\timestep}(\state, \action) | \leq E_\episode^\timestep(\state, \action)
\]
\end{lemma}

\begin{proof}
For $\timestep=\Horizon$ the result holds trivially. Now assuming it holds for $\timestep+1$, we consider step $\timestep$:
\begin{align*}
&| \Qfun{\EstMDP\cup\hat{\reward}}{\Policy}{\timestep}(\state, \action) - \Qfun{\EstMDP\cup\reward}{\Policy}{\timestep}(\state, \action) | \\
\leq &| \hat{\reward}(\state,\action) - \reward(\state,\action) | + \sum_{\state'} \EstTransitionModel(\state'|\state,\action) \sum_{\action'} \Policy(\action'|\state') | \Qfun{\EstMDP\cup\hat{\reward}}{\Policy}{\timestep+1}(\state', \action') - \Qfun{\EstMDP\cup\reward}{\Policy}{\timestep+1}(\state', \action') | \\
\leq &| \hat{\reward}(\state,\action) - \reward(\state,\action) | + \sum_{\state'} \EstTransitionModel(\state'|\state,\action) \max_{\action'} | \Qfun{\EstMDP\cup\hat{\reward}}{\Policy}{\timestep+1}(\state', \action') - \Qfun{\EstMDP\cup\reward}{\Policy}{\timestep+1}(\state', \action') | \\
\leq &| \hat{\reward}(\state,\action) - \reward(\state,\action) | + \sum_{\state'} \EstTransitionModel(\state'|\state,\action) \max_{\action'} E_\episode^{\timestep+1}(\state', \action')
= E_\episode^{\timestep}(\state, \action)
\end{align*}
\end{proof}

We can now combine the previous two lemmas to show that $E$ is indeed an upper bound on the error we want to reduce. This implies correctness of \AlgNameShort Greedy, which the following lemma formalizes.

\begin{lemma}[Correctness of \AlgNameShort Greedy]\label{aceirl_correctness}
If \AlgNameShort Greedy stops in episode $\episode$, after sampling $n$ samples, i.e., $E_\episode^0(\state_0, \Policy_{\episode+1}(\state_0)) \leq \frac{\epsilon}{4}$, then it fulfills \Cref{def:correct}.
\end{lemma}

\begin{proof}
Let us define the error
\[
e_\episode^{\timestep}(\state, \action) \coloneqq | \Qfun{\MDP\cup\reward}{\Policy^*}{\timestep}(\state, \action) - \Qfun{\MDP\cup\reward}{\hat{\Policy}^*}{\timestep}(\state, \action)  |
\]
where $\Policy^*$ is the true optimal policy in $\MDP\cup\reward$, and $\hat{\Policy}^*$ is the optimal policy in $\EstMDP\cup\hat{\reward}$, i.e., in the estimated MDP using the inferred reward function. Then,
\begin{align*}
&e_\episode^{\timestep}(\state, \action) = | \Qfun{\MDP\cup\reward}{\Policy^*}{\timestep}(\state, \action) - \Qfun{\MDP\cup\reward}{\hat{\Policy}^*}{\timestep}(\state, \action) \pm \Qfun{\EstMDP\cup\reward}{\Policy^*}{\timestep}(\state, \action) \pm \Qfun{\EstMDP\cup\reward}{\hat{\Policy}^*}{\timestep}(\state, \action) | \\
&\leq \underbrace{| \Qfun{\MDP\cup\reward}{\Policy^*}{\timestep}(\state, \action) - \Qfun{\EstMDP\cup\reward}{\Policy^*}{\timestep}(\state, \action) |}_{\leq E_\episode^{\timestep}(\state, \action)} + | \Qfun{\EstMDP\cup\reward}{\Policy^*}{\timestep}(\state, \action) - \Qfun{\EstMDP\cup\reward}{\hat{\Policy}^*}{\timestep}(\state, \action) | + \underbrace{| \Qfun{\EstMDP\cup\reward}{\hat{\Policy}^*}{\timestep}(\state, \action) - \Qfun{\MDP\cup\reward}{\hat{\Policy}^*}{\timestep}(\state, \action) |}_{\leq E_\episode^{\timestep}(\state, \action)} \\
&\leq 2 E_\episode^{\timestep}(\state, \action) + | \Qfun{\EstMDP\cup\reward}{\Policy^*}{\timestep}(\state, \action) - \Qfun{\EstMDP\cup\reward}{\hat{\Policy}^*}{\timestep}(\state, \action) |
\end{align*}
where, we used \Cref{lemma:ucrl_error_estmdp}.

Let us consider the remaining term $| \Qfun{\EstMDP\cup\reward}{\Policy^*}{\timestep}(\state, \action) - \Qfun{\EstMDP\cup\reward}{\hat{\Policy}^*}{\timestep}(\state, \action) |$ in two steps. First, we have:
\begin{align*}
\Qfun{\EstMDP\cup\reward}{\Policy^*}{\timestep}(\state, \action) - \Qfun{\EstMDP\cup\reward}{\hat{\Policy}^*}{\timestep}(\state, \action)
\leq& \underbrace{\Qfun{\EstMDP\cup\reward}{\Policy^*}{\timestep}(\state, \action) - \Qfun{\EstMDP\cup\hat{\reward}}{\Policy^*}{\timestep}(\state, \action)}_{\leq E_\episode^{\timestep}(\state, \action)} + \underbrace{\Qfun{\EstMDP\cup\hat{\reward}}{\Policy^*}{\timestep}(\state, \action) - \Qfun{\EstMDP\cup\hat{\reward}}{\hat{\Policy}^*}{\timestep}(\state, \action)}_{\leq 0} + \\ &+ \underbrace{\Qfun{\EstMDP\cup\hat{\reward}}{\hat{\Policy}^*}{\timestep}(\state, \action) - \Qfun{\EstMDP\cup\reward}{\hat{\Policy}^*}{\timestep}(\state, \action)}_{\leq E_\episode^{\timestep}(\state, \action)}
\leq 2 E_\episode^{\timestep}(\state, \action),
\end{align*}
where we used \Cref{lemma:ucrl_error_estreward} and the fact that $\hat{\Policy}^*$ is optimal in the MDP $\EstMDP\cup\hat{\reward}$. Second, we have:
\begin{align*}
\Qfun{\EstMDP\cup\reward}{\hat{\Policy}^*}{\timestep}(\state, \action) - \Qfun{\EstMDP\cup\reward}{\Policy^*}{\timestep}(\state, \action)
\leq& \underbrace{\Qfun{\EstMDP\cup\reward}{\hat{\Policy}^*}{\timestep}(\state, \action) - \Qfun{\MDP\cup\reward}{\hat{\Policy}^*}{\timestep}(\state, \action)}_{\leq E_\episode^{\timestep}(\state, \action)} + \underbrace{\Qfun{\MDP\cup\reward}{\hat{\Policy}^*}{\timestep}(\state, \action) - \Qfun{\MDP\cup\reward}{\Policy^*}{\timestep}(\state, \action)}_{\leq 0} + \\
&+ \underbrace{\Qfun{\MDP\cup\reward}{\Policy^*}{\timestep}(\state, \action) - \Qfun{\EstMDP\cup\reward}{\Policy^*}{\timestep}(\state, \action)}_{\leq E_\episode^{\timestep}(\state, \action)}
\leq 2 E_\episode^{\timestep}(\state, \action),
\end{align*}
where we used \Cref{lemma:ucrl_error_estmdp} and the fact that $\Policy^*$ is optimal in the MDP $\MDP\cup\reward$. Overall, we find that
\[
| \Qfun{\EstMDP\cup\reward}{\Policy^*}{\timestep}(\state, \action) - \Qfun{\EstMDP\cup\reward}{\hat{\Policy}^*}{\timestep}(\state, \action) | \leq 2 E_\episode^{\timestep}(\state, \action),
\]
and consequently,
\[
e_\episode^{\timestep}(\state, \action) \leq 4 E_\episode^{\timestep}(\state, \action).
\]

Hence, if $E_\episode^0(\state_0, \Policy_{\episode+1}(\state_0)) \leq \frac{\epsilon}{4}$, we have for all $\action\in\ActionSpace$:
\[
e_\episode^{0}(\state_0, \action) \leq \epsilon
\]
which implies correctness according to \Cref{def:correct}.
\end{proof}

Next, we will analyze the sample complexity of \AlgNameShort Greedy. Let us first define pseudo-counts that will be crucial to deal with the uncertainty of the transition dynamics in our analysis. This is similar to the analysis of UCRL for reward-free exploration by \citet{kaufmann2021adaptive}.

\begin{definition}
We define the \emph{pseudo-counts} of visiting a specific state action pair at timestep $\timestep$ within the first $\episode$ iterations as
\[
\bar{n}^\timestep_\episode(\state, \action) \coloneqq \sum_{i=1}^\episode \Occupancy{\MDP}{\Policy_i}{0}{\timestep}(\state, \action | \state_0),
\]
where $\Policy_i$ is the exploration policy in episode $i$.
\end{definition}

The following lemma allows us to introduce the pseudo-counts when considering the contraction of the reward confidence intervals.

\begin{lemma}\label{lemma:psuedo_counts_ci}
With probability at least $1-\frac{\delta}{2}$ for all $\state, \action, \timestep, \episode \in \StateSpace \times \ActionSpace \times [\Horizon] \times \mathbb{N}^+$, we have:
\[
\min \left(\frac{2 \ell^\timestep_\episode(\state, \action)}{{n^\timestep_\episode}(\state,\action)}, 1\right) \leq \frac{8 \bar{\ell}^\timestep_\episode(\state, \action)}{\max\left( \bar{n}^\timestep_\episode(\state, \action) , 1 \right)}
\]
where $\bar{\ell}^\timestep_k(\state, \action) = \log\left( 24 \StateSpaceSize \ActionSpaceSize \Horizon ({\bar{n}^\timestep_k}(\state, \action))^2 / \delta \right)$.
\end{lemma}

\begin{proof}
This result adapts Lemma 7 by \citet{kaufmann2021adaptive} to our setting.

By Lemma 10 in \citet{kaufmann2021adaptive}, we have with probability at least $1-\frac{\delta}{2}$:
\[
n^\timestep_\episode(\state, \action) \geq \frac12 \bar{n}^\timestep_\episode(\state, \action) - \beta_{\mathrm{cnt}}(\delta),
\]
where $\beta_{\mathrm{cnt}}(\delta) = \log(2\StateSpaceSize\ActionSpaceSize\Horizon/\delta)$.

We distinguish two cases. First let $\beta_{\mathrm{cnt}}(\delta) \leq \frac14 \bar{n}^\timestep_\episode(\state, \action)$. Then $n^\timestep_\episode(\state, \action) \geq \frac14 \bar{n}^\timestep_\episode(\state, \action)$, and
\begin{align*}
&\min \left(\frac{2 \ell^\timestep_\episode(\state, \action)}{{n^\timestep_\episode}(\state,\action)}, 1\right) 
\leq \frac{2 \ell^\timestep_\episode(\state, \action)}{\max({n^\timestep_\episode}(\state,\action), 1)}
= \frac{2 \log(24 \StateSpaceSize \ActionSpaceSize \Horizon ({n^\timestep_\episode}(\state, \action))^2 / \delta)}{\max({n^\timestep_\episode}(\state,\action), 1)} \\
&\leq \frac{2 \log(24 \StateSpaceSize \ActionSpaceSize \Horizon ({\bar{n}^\timestep_\episode}(\state, \action)/4)^2 / \delta)}{({\bar{n}^\timestep_\episode}(\state,\action) / 4)}
\leq \frac{8 \bar{\ell}^\timestep_\episode(\state, \action)}{\max({\bar{n}^\timestep_\episode}(\state,\action), 1)}
\end{align*}
where we use that $\log(24 \StateSpaceSize \ActionSpaceSize \Horizon x^2 / \delta) / x$ is non-increasing for $x > 1$, and $\log(24 \StateSpaceSize \ActionSpaceSize \Horizon x^2 / \delta)$ is non-decreasing and $\beta_{\mathrm{cnt}}(\delta) \geq 1$.

Now consider let $\beta_{\mathrm{cnt}}(\delta) > \frac14 \bar{n}^\timestep_\episode(\state, \action)$. Then,
\begin{align*}
&\min \left(\frac{2 \ell^\timestep_\episode(\state, \action)}{{n^\timestep_\episode}(\state,\action)}, 1\right) 
\leq 1 < 4 \frac{\beta_{\mathrm{cnt}}(\delta)}{\max({\bar{n}^\timestep_\episode}(\state,\action), 1)} \leq \frac{4 \bar{\ell}^\timestep_\episode(\state, \action)}{\max({\bar{n}^\timestep_\episode}(\state,\action), 1)}
\end{align*}
where we used that $\ell^\timestep_\episode(\state, \action) = \log\left( 24 \StateSpaceSize \ActionSpaceSize \Horizon ({n^\timestep_\episode}(\state, \action))^2 / \delta \right) = \beta_{\mathrm{cnt}}(\delta) + \log \left( 6 {n^\timestep_\episode}(\state, \action))^2 \right) \geq \beta_{\mathrm{cnt}}(\delta)$.
\end{proof}

The final lemma we need shows relates the error upper bound which is defined using our estimated transition model to a similar quantity defined using the (unknown) real transitions. 

\begin{lemma}\label{lemma:ucrl_error_real_transitions}
Under the good event $\GoodEvent$, we have for any $\state,\action,\timestep:$
\[
E_\episode^{\timestep}(\state, \action)
\leq 2 \rewardci_\episode^\timestep(\state,\action) + \sum_{\state'} \TransitionModel(\state'|\state,\action) \max_{\action'} E_\episode^{\timestep+1}(\state', \action')
\]
where $\TransitionModel$ is the true transition model that we do not know.
\end{lemma}

\begin{proof}
First note that $E_\episode^{\timestep}(\state, \action) \leq H$ by definition. Now, consider:
\begin{align*}
E_\episode^{\timestep}(\state, \action)
&\leq \rewardci_\episode^\timestep(\state,\action) + \sum_{\state'} \EstTransitionModel(\state'|\state,\action) \max_{\action'} E_\episode^{\timestep+1}(\state', \action') \\
&= \rewardci_\episode^\timestep(\state,\action) + \sum_{\state'} (\EstTransitionModel(\state'|\state,\action) - \TransitionModel(\state'|\state,\action) + \TransitionModel(\state'|\state,\action)) \max_{\action'} E_\episode^{\timestep+1}(\state', \action') \\
&= \rewardci_\episode^\timestep(\state,\action) + \sum_{\state'} \underbrace{(\EstTransitionModel(\state'|\state,\action) - \TransitionModel(\state'|\state,\action))\max_{\action'} E_\episode^{\timestep+1}(\state', \action')}_{\leq \rewardci_\episode^\timestep(\state,\action)} + \sum_{\state'} \TransitionModel(\state'|\state,\action)) \max_{\action'} E_\episode^{\timestep+1}(\state', \action') \\
&\leq 2 \rewardci_\episode^\timestep(\state,\action) + \sum_{\state'} \TransitionModel(\state'|\state,\action) \max_{\action'} E_\episode^{\timestep+1}(\state', \action') \\
\end{align*}
where we used the good event and the fact that $\rewardci_\episode^\timestep$ can only shrink over episodes.
\end{proof}

Finally, we can analyze the sample complexity of \AlgNameShort Greedy.

\begin{theorem}[\AlgNameShort Greedy Sample Complexity (problem independent)]
\label{thm:simple_sample_complexity}
\AlgNameShort Greedy terminates with an ($\epsilon$, $\delta$, $n$)-correct solution, with
\[
n \leq \tilde{\mathcal{O}}\left( \frac{\Horizon^5 \Rmax^2 \StateSpaceSize \ActionSpaceSize}{\epsilon^2} \right).
\]
\end{theorem}

\begin{proof}
\Cref{aceirl_correctness} shows that if \AlgNameShort Greedy terminates, then it returns a ($\epsilon$, $\delta$, $n$)-correct solution. So, we need to show that it terminates within $\tau$ iterations and bound $\tau$.

Let us consider the average error, defined by
\begin{align*}
q_\episode^\timestep &\coloneqq \sum_{\state,\action} \Occupancy{\MDP}{\Policy_{\episode+1}}{0}{\timestep}(\state, \action | \state_0) E_\episode^\timestep(\state, \action) \\
&\overset{(a)}{\leq} \sum_{\state,\action} \Occupancy{\MDP}{\Policy_{\episode+1}}{0}{\timestep}(\state, \action | \state_0) \bigl( 2 \rewardci_\episode^\timestep(\state,\action) + \sum_{\state'} \TransitionModel(\state'|\state,\action) \max_{\action'} E_\episode^{\timestep+1}(\state', \action') \bigr) \\
&= \sum_{\state,\action} \Occupancy{\MDP}{\Policy_{\episode+1}}{0}{\timestep}(\state, \action | \state_0) \bigl( 2 \rewardci_\episode^\timestep(\state,\action) + \sum_{\state'} \TransitionModel(\state'|\state,\action) \sum_{\action'} \Policy_{\episode+1}(\action'|\state') E_\episode^{\timestep+1}(\state', \action') \bigr) \\
&= 2 \sum_{\state,\action} \Occupancy{\MDP}{\Policy_{\episode+1}}{0}{\timestep}(\state, \action | \state_0) \rewardci_\episode^\timestep(\state,\action) + q_\episode^{\timestep+1}
\end{align*}
where we used \Cref{lemma:ucrl_error_real_transitions} in step (a). Unrolling the recursion, results in:
\[
q_\episode^\timestep \leq 2 \sum_{\timestep'=\timestep}^\Horizon \sum_{\state,\action} \Occupancy{\MDP}{\Policy_{\episode+1}}{0}{\timestep'}(\state, \action | \state_0) \rewardci_\episode^{\timestep'}(\state,\action)
\]

If the algorithm terminates at $\tau$, we have for each $\episode < \tau$, and $\state,\action,\timestep \in \StateSpace \times \ActionSpace \times [\Horizon]$:
$\epsilon < 4 E_\episode^0(\state_0, \Policy_{\episode+1}(\state_0))$. We have $q_\episode^0 = E_\episode^0(\state_0, \Policy_{\episode+1}(\state_0))$; therefore, as long we haven't stopped, we have $\epsilon \leq 4 q_\episode^0$. Writing out this inequality, yields:
\resizebox{\linewidth}{!}{
  \begin{minipage}{\linewidth}
\begin{align*}
\epsilon \leq 4 q_\episode^0 \leq 8 \sum_{\timestep=0}^\Horizon \sum_{\state,\action} \Occupancy{\MDP}{\Policy_{\episode+1}}{0}{\timestep}(\state, \action | \state_0) \rewardci_\episode^{\timestep}(\state,\action)
\leq 4 \Horizon \Rmax \sum_{\timestep=0}^\Horizon \sum_{\state,\action} \Occupancy{\MDP}{\Policy_{\episode+1}}{0}{\timestep}(\state, \action | \state_0) \sqrt{\frac{8 \log(12 \StateSpaceSize \ActionSpaceSize \Horizon ({n^\timestep_\episode}(\state,\action))^2/\delta)}{\max({n^\timestep_\episode}(\state,\action), 1)}}
\end{align*}
\end{minipage}}

Using \Cref{lemma:psuedo_counts_ci}, we can relate this to the pseudo-counts
\begin{align*}
\epsilon
&< 4 \Horizon \Rmax \sum_{\timestep=0}^\Horizon \sum_{\state,\action} \Occupancy{\MDP}{\Policy_{\episode+1}}{0}{\timestep}(\state, \action | \state_0) \sqrt{\frac{8 \log(12 \StateSpaceSize \ActionSpaceSize \Horizon ({\bar{n}^\timestep_\episode}(\state,\action))^2/\delta)}{\max({\bar{n}^\timestep_\episode}(\state,\action), 1)}} \\
&\leq 4 \Horizon \Rmax \sum_{\timestep=0}^\Horizon \sum_{\state,\action} \Occupancy{\MDP}{\Policy_{\episode+1}}{0}{\timestep}(\state, \action | \state_0) \sqrt{\frac{8 \log(12 \StateSpaceSize \ActionSpaceSize \Horizon \episode^2 / \delta)}{\max({\bar{n}^\timestep_\episode}(\state,\action), 1)}}
\end{align*}

Summing the inequality over $\episode=0, \dots T$ with $T < \tau$, we obtain
\begin{align*}
\epsilon (T+1)
&\leq 4 \Horizon \Rmax \sqrt{8 \log(12 \StateSpaceSize \ActionSpaceSize \Horizon T^2 / \delta)} \sum_{\timestep=0}^\Horizon \sum_{\state,\action} \sum_{\episode=1}^T \Occupancy{\MDP}{\Policy_{\episode+1}}{0}{\timestep}(\state, \action | \state_0) \frac{1}{\sqrt{\max({\bar{n}^\timestep_k}(\state,\action), 1)}} \\
&= 4 \Horizon \Rmax \sqrt{8 \log(12 \StateSpaceSize \ActionSpaceSize \Horizon T^2 / \delta)} \sum_{\timestep=0}^\Horizon \sum_{\state,\action} \sum_{\episode=1}^T \frac{\bar{n}_\timestep^{\episode+1}(\state, \action) - \bar{n}_\timestep^{\episode}(\state, \action)}{\sqrt{\max({\bar{n}^\timestep_k}(\state,\action), 1)}}
\end{align*}
where we used the definition of the pseudo-counts in the last equality. Using Lemma 19 by \cite{jaksch2010near}, we can further bound the sum in $\episode$:
\begin{align*}
\epsilon (T+1)
&= 4 \Horizon \Rmax \sqrt{8 \log(12 \StateSpaceSize \ActionSpaceSize \Horizon T^2 / \delta)} \sum_{\timestep=0}^\Horizon \sum_{\state,\action} \sqrt{\bar{n}_\timestep^{T+1}(\state, \action)} \\
&\leq 4 \Horizon \Rmax \sqrt{8 \log(12 \StateSpaceSize \ActionSpaceSize \Horizon T^2 / \delta)} \sqrt{\StateSpaceSize\ActionSpaceSize} \sum_{\timestep=0}^\Horizon \sqrt{\sum_{\state,\action} \bar{n}_\timestep^{T+1}(\state, \action)} \\
&= 4 \Horizon^2 \Rmax \sqrt{8 \log(12 \StateSpaceSize \ActionSpaceSize \Horizon T^2 / \delta)} \sqrt{\StateSpaceSize\ActionSpaceSize} \sqrt{T+1}
\end{align*}

It follows that
\begin{align*}
\epsilon \sqrt{T+1} &\leq 4 \Horizon^2 \Rmax \sqrt{8 \StateSpaceSize\ActionSpaceSize \log(12 \StateSpaceSize \ActionSpaceSize \Horizon T^2 / \delta)} \\
\epsilon^2 \tau &\leq 128 \Horizon^4 \Rmax^2 \StateSpaceSize\ActionSpaceSize \log(12 \StateSpaceSize \ActionSpaceSize \Horizon (\tau-1)^2 / \delta)
\end{align*}
setting $\tau = T+1$.

For large enough $\tau$, this inequality cannot hold because $\sqrt{T+1}$ on the l.h.s grows faster than $\log(\tau)$ on the r.h.s. Hence, the stopping time $\tau$ is finite. Further, we can apply Lemma 15 by \cite{kaufmann2021adaptive}, and follow that
\begin{align*}
\tau \leq \tilde{\mathcal{O}} \left( \frac{\Horizon^4 \Rmax^2 \StateSpaceSize \ActionSpaceSize}{\epsilon^2} \right)
\end{align*}
If we observe $H$ samples in each iteration, i.e., $N_E=1$, we get a sample complexity of
\begin{align*}
n \leq \tilde{\mathcal{O}} \left( \frac{\Horizon^5 \Rmax^2 \StateSpaceSize \ActionSpaceSize}{\epsilon^2} \right)
\end{align*}
\end{proof}

\subsection{Sample Complexity of \AlgNameShort in Unknown Environments (Problem Dependent)}
\label{app:sample-complexity2}

For the problem dependent analysis, we will need this additional lemma also used by \citet{kakade2002approximately}.

\begin{lemma}[Lemma 6.1 by \citet{kakade2002approximately}]\label{lemma:sum_of_losses}
For any policy $\Policy$:
\[
\Valfun{\MDP\cup\reward}{\Policy^*}{\timestep}(\state) - \Valfun{\MDP\cup\reward}{\Policy}{\timestep}(\state)
= - \sum_{\state',\action'} \sum_{\timestep'=\timestep}^\Horizon \Occupancy{\MDP}{\Policy}{\timestep}{\timestep'}(\state',\action';\state) A_{\MDP\cup\reward}^{*,\timestep'}(\state', \action')
\]
\end{lemma}

\begin{proof}
\begin{align*}
&\Valfun{\MDP\cup\reward}{*}{\timestep}(\state) - \Valfun{\MDP\cup\reward}{\Policy}{\timestep}(\state) \\
=& \sum_\action \Policy_\timestep^*(\action|\state) \left( \reward_\timestep(\state, \action) + \sum_{\state'} \TransitionModel(\state'|\state,\action) \Valfun{\MDP\cup\reward}{*}{\timestep+1}(\state') \right) \\ 
&- \sum_\action \Policy_\timestep(\action|\state) \left( \reward_\timestep(\state, \action) + \sum_{\state'} \TransitionModel(\state'|\state,\action) \Valfun{\MDP\cup\reward}{\Policy}{\timestep+1}(\state') \right) \pm \sum_{\action,\state'} \Policy_\timestep(\action|\state) \TransitionModel(\state'|\state,\action) \Valfun{\MDP\cup\reward}{*}{\timestep+1}(\state') \\
=& \sum_\action (\Policy_\timestep^*(\action|\state) - \Policy_\timestep(\action|\state)) \reward(\state,\action) + \sum_{\action,\state'} (\Policy_\timestep^*(\action|\state) - \Policy_\timestep(\action|\state)) \TransitionModel(\state'|\state,\action) \Valfun{\MDP\cup\reward}{*}{\timestep+1}(\state') \\
&+ \sum_{\action,\state'} \Policy_\timestep(\action|\state) \TransitionModel(\state'|\state,\action) (\Valfun{\MDP\cup\reward}{*}{\timestep+1}(\state) - \Valfun{\MDP\cup\reward}{\Policy}{\timestep+1}(\state)) \\
=& - \sum_{\action} \Policy(\action|\state) A_{\MDP\cup\reward}^{*,\timestep}(\state, \action)
+ \sum_{\action,\state'} \Policy_\timestep(\action|\state) \TransitionModel(\state'|\state,\action) (\Valfun{\MDP\cup\reward}{*}{\timestep+1}(\state) - \Valfun{\MDP\cup\reward}{\Policy}{\timestep+1}(\state))
\end{align*}
Unrolling the recursion yields the result.
\end{proof}
We can now start with the analysis. First, we define the policy confidence set, and show that it indeed contains the relevant policies under the good event.
\begin{definition}\label{def:policyci2}
We define the \emph{policy confidence set} as
\[
\policyci_\episode = \{ \Policy | \Valfun{\EstMDP\cup\hat{\reward}}{*}{}(\state_0) - \Valfun{\EstMDP\cup\hat{\reward}}{\Policy}{}(\state_0) \leq 10 \epsilon_\episode \}
\]
where $\hat{\reward} = \irlalg(\RecoveredFeasibleSet)$ is the reward estimated using an IRL algorithm $\irlalg$.
We choose $\epsilon_\episode$ recursively by solving the optimization problem
\[
\epsilon_\episode = \max_{\Policy \in \policyci_{\episode-1}} \sum_{\timestep=0}^{\Horizon} \sum_{\state',\action'}  \Occupancy{\EstMDP}{\Policy}{0}{\timestep}(\state',\action';\state_0) \rewardci^{\timestep}_\episode(\state',\action')
\]
starting with $\epsilon_0 = \frac{1}{10} \Horizon$.
\end{definition}

The following lemma will help us to deal with uncertainty about the transition dynamics.
\begin{lemma}\label{lemma_policyci}
Under the good event $\GoodEvent$, if $\Policy \in \policyci_\episode$, then:
\begin{align*}
&|\Valfun{\EstMDP\cup\hat{\reward}}{\Policy}{\timestep}(\state) - \Valfun{\MDP\cup\hat{\reward}}{\Policy}{\timestep}(\state)| \leq \epsilon_\episode \\
&|\Valfun{\MDP\cup\hat{\reward}}{*}{\timestep}(\state) - \Valfun{\EstMDP\cup\hat{\reward}}{*}{\timestep}(\state)| \leq \epsilon_\episode
\end{align*}
\end{lemma}

\begin{proof}
First by \Cref{dynamics_simulation_lemma_1}:
\begin{align*}
|\Valfun{\EstMDP\cup\reward}{\Policy}{\timestep}(\state) - \Valfun{\MDP\cup\reward}{\Policy}{\timestep}(\state)|
&\leq \sum_{\timestep'=\timestep}^\Horizon \sum_{\state',\action',\state''} \Occupancy{\EstMDP}{\Policy}{\timestep}{\timestep'}(\state';\state) \Policy_{\timestep'}(\action' | \state') | \EstTransitionModel(\state'' | \state', \action') - \TransitionModel(\state'' | \state', \action') | \Valfun{\MDP\cup\reward}{\Policy}{\timestep'+1}(\state'') \\
&\leq \sum_{\timestep'=\timestep}^\Horizon \sum_{\state',\action'} \Occupancy{\EstMDP}{\Policy}{\timestep}{\timestep'}(\state';\state) \Policy_{\timestep'}(\action' | \state') \rewardci_\episode(\state', \action') \leq \epsilon_\episode
\end{align*}

Then, by \Cref{dynamics_simulation_lemma_2}:
\begin{align*}
\Valfun{\MDP\cup\reward}{*}{\timestep}(\state) - \Valfun{\EstMDP\cup\reward}{*}{\timestep}(\state)
&\leq \sum_{\timestep'=\timestep} \sum_{\state',\action',\state''} \Occupancy{\EstMDP}{\Policy^*}{\timestep}{\timestep'}(\state';\state) \Policy_{\timestep'}^*(\action'|\state') (\TransitionModel(\state''|\state',\action') - \EstTransitionModel(\state''|\state',\action')) \Valfun{\MDP\cup\reward}{*}{\timestep}(\state'') \\
&\leq \sum_{\timestep'=\timestep} \sum_{\state',\action'} \Occupancy{\EstMDP}{\Policy^*}{\timestep}{\timestep'}(\state';\state) \Policy_{\timestep'}^*(\action'|\state') \rewardci_\episode(\state',\action') \leq \epsilon_\episode
\end{align*}

And, similarly
\begin{align*}
\Valfun{\EstMDP\cup\reward}{*}{\timestep}(\state) - \Valfun{\MDP\cup\reward}{*}{\timestep}(\state)
&\leq \sum_{\timestep'=\timestep} \sum_{\state',\action',\state''} \Occupancy{\EstMDP}{\hat{\Policy}^*}{\timestep}{\timestep'}(\state';\state) \hat{\Policy}_{\timestep'}^*(\action'|\state') (\EstTransitionModel(\state''|\state',\action') - \TransitionModel(\state''|\state',\action')) \Valfun{\EstMDP\cup\reward}{*}{\timestep}(\state'') \\
&\leq \sum_{\timestep'=\timestep} \sum_{\state',\action'} \Occupancy{\EstMDP}{\hat{\Policy}^*}{\timestep}{\timestep'}(\state';\state) \hat{\Policy}_{\timestep'}^*(\action'|\state') \rewardci_\episode(\state',\action') \leq \epsilon_\episode
\end{align*}
\end{proof}

Now we show that the relevant policies are always in the policy confidence set, conditioned on the good event.
\begin{lemma}\label{policyci_lemma_1}
Conditioned the good event $\GoodEvent$, if $\Policy^*, \hat{\Policy}^* \in \policyci_{\episode-1}$, then $\Policy^* \in \policyci_{\episode}$.
\end{lemma}

\begin{proof}
Let $\reward \in \ExactFeasibleSet$. Then
\begin{align*}
&\Valfun{\EstMDP\cup\hat{\reward}_\episode}{*}{\timestep}(\state) - \Valfun{\EstMDP\cup\hat{\reward}_\episode}{\Policy^*}{\timestep}(\state)
= \Valfun{\EstMDP\cup\hat{\reward}_\episode}{*}{\timestep}(\state) - \Valfun{\EstMDP\cup\reward}{*}{\timestep}(\state) + \Valfun{\EstMDP\cup\reward}{*}{\timestep}(\state) - \Valfun{\EstMDP\cup\hat{\reward}_\episode}{\Policy^*}{\timestep}(\state) \\
\overset{(a)}{\leq}& \sum_{\timestep'=\timestep}^{\Horizon} \sum_{\state',\action'} \Occupancy{\EstMDP}{\Policy^*}{\timestep}{\timestep'}(\state', \action'|\state) \rewardci_\episode^{\timestep'}(\state',\action') + \sum_{\timestep'=\timestep}^{\Horizon} \sum_{\state',\action'} \Occupancy{\EstMDP}{\Policy^*}{\timestep}{\timestep'}(\state', \action'|\state) \rewardci_\episode^{\timestep'}(\state',\action')
\overset{(b)}{\leq} 2 \epsilon_\episode
\end{align*}
where (a) uses \Cref{qfun_occupancy_lemma}, \Cref{simulation_lemma_1} and \Cref{cor:rewarci}, (b) uses that $\Policy^* \in \policyci_{\episode-1}$ and the definition of $\epsilon_\episode$.
Hence,
\[
\max_\state \left( \Valfun{\EstMDP\cup\hat{\reward}_\episode}{*}{\timestep}(\state) - \Valfun{\EstMDP\cup\hat{\reward}_\episode}{\Policy^*}{\timestep}(\state) \right) \leq 2 \epsilon_\episode \leq 10 \epsilon_\episode
\]
and therefore $\Policy^* \in \policyci_\episode$.
\end{proof}

\begin{lemma}\label{policyci_lemma_2}
Conditioned on the good event $\GoodEvent$, for every policy $\Policy$ and episodes $\episode' > \episode$, there exists $\hat{\reward}_{\episode'} \in {\RecoveredFeasibleSet}_{\episode'}$, such that:
\[
\max_\state \left( \Valfun{\MDP\cup\hat{\reward}_{\episode'}}{\Policy}{\timestep}(\state) - \Valfun{\MDP\cup\hat{\reward}_{\episode}}{\Policy}{\timestep}(\state) \right) \leq 4 \epsilon_\episode
\]
\end{lemma}

\begin{proof}
Similarly to the proof of the previous lemma, we have
\begin{align*}
&\Valfun{\EstMDP\cup\hat{\reward}_{\episode'}}{\Policy}{\timestep}(\state) - \Valfun{\EstMDP\cup\hat{\reward}_\episode}{\Policy}{\timestep}(\state)
= \Valfun{\EstMDP\cup\hat{\reward}_{\episode'}}{\Policy}{\timestep}(\state) - \Valfun{\EstMDP\cup\reward}{\Policy}{\timestep}(\state) + \Valfun{\EstMDP\cup\reward}{\Policy}{\timestep}(\state) - \Valfun{\EstMDP\cup\hat{\reward}_\episode}{\Policy}{\timestep}(\state) \\
\leq& \sum_{\timestep'=\timestep}^{\Horizon} \sum_{\state',\action'} \Occupancy{\EstMDP}{\Policy}{\timestep}{\timestep'}(\state', \action'|\state) \rewardci_{\episode'}^{\timestep'}(\state',\action') + \sum_{\timestep'=\timestep}^{\Horizon} \sum_{\state',\action'} \Occupancy{\EstMDP}{\Policy}{\timestep}{\timestep'}(\state', \action'|\state) \rewardci_\episode^{\timestep'}(\state',\action')
\leq 2 \epsilon_\episode
\end{align*}
where we use that the confidence intervals are shrinking with increasing episode number, i.e., $\epsilon_{\episode'} \leq \epsilon_{\episode}$.

By combining this with \Cref{lemma_policyci}, we get the result:
\begin{align*}
&\max_\state \left( \Valfun{\MDP\cup\hat{\reward}_{\episode'}}{\Policy}{\timestep}(\state) - \Valfun{\MDP\cup\hat{\reward}_{\episode}}{\Policy}{\timestep}(\state) \right) \\
=& \max_\state \Bigl( \underbrace{\Valfun{\MDP\cup\hat{\reward}_{\episode'}}{\Policy}{\timestep}(\state) - \Valfun{\EstMDP\cup\hat{\reward}_{\episode'}}{\Policy}{\timestep}(\state)}_{\leq \epsilon_\episode} + \underbrace{\Valfun{\EstMDP\cup\hat{\reward}_{\episode'}}{\Policy}{\timestep}(\state) - \Valfun{\EstMDP\cup\hat{\reward}_{\episode}}{\Policy}{\timestep}(\state)}_{\leq 2\epsilon_\episode } + \underbrace{\Valfun{\EstMDP\cup\hat{\reward}_{\episode}}{\Policy}{\timestep}(\state) - \Valfun{\MDP\cup\hat{\reward}_{\episode}}{\Policy}{\timestep}(\state)}_{\leq \epsilon_\episode } \Bigr) \leq 4 \epsilon_\episode
\end{align*}
\end{proof}

\begin{lemma}
Under the good event $\GoodEvent$, if $\hat{\Policy}^*_\episode, \Policy \in \policyci_{\episode-1}$ and $\Policy \notin \policyci_\episode$, then the policy $\Policy$ is suboptimal for some reward $\hat{\reward}_{\episode'} \in \RecoveredFeasibleSetIterPrime$ for all $\episode' \geq \episode$.
\end{lemma}

\begin{proof}
We can observe that
\begin{align*}
&\Valfun{\MDP\cup\hat{\reward}_{\episode'}}{\Policy}{\timestep}(\state_0) - \Valfun{\MDP\cup\hat{\reward}_{\episode'}}{*}{\timestep}(\state_0)
= \Valfun{\MDP\cup\hat{\reward}_{\episode'}}{\Policy}{\timestep}(\state_0) - \Valfun{\MDP\cup\hat{\reward}_{\episode'}}{\hat{\Policy}^*_\episode}{\timestep}(\state_0) \\
=& \underbrace{\Valfun{\MDP\cup\hat{\reward}_{\episode'}}{\Policy}{\timestep}(\state_0) - \Valfun{\MDP\cup\hat{\reward}_{\episode}}{\Policy}{\timestep}(\state_0)}_{\overset{(a)}{\leq} 4 \epsilon_\episode}
+ \underbrace{\Valfun{\MDP\cup\hat{\reward}_{\episode}}{\Policy}{\timestep}(\state_0) - \Valfun{\EstMDP\cup\hat{\reward}_{\episode}}{\Policy}{\timestep}(\state_0)}_{\overset{(b)}{\leq} \epsilon_\episode} \\
&+ \underbrace{\Valfun{\EstMDP\cup\hat{\reward}_{\episode}}{\Policy}{\timestep}(\state_0) - \Valfun{\EstMDP\cup\hat{\reward}_\episode}{\hat{\Policy}^*_\episode}{\timestep}(\state_0)}_{\overset{(c)}{>} 10 \epsilon_\episode}
+ \underbrace{\Valfun{\EstMDP\cup\hat{\reward}_\episode}{\hat{\Policy}^*_\episode}{\timestep}(\state_0) - \Valfun{\MDP\cup\hat{\reward}_\episode}{\hat{\Policy}^*_\episode}{\timestep}(\state_0)}_{\overset{(b)}{\leq} \epsilon_\episode} \\
&+ \underbrace{\Valfun{\MDP\cup\hat{\reward}_\episode}{\hat{\Policy}^*_\episode}{\timestep}(\state_0) - \Valfun{\MDP\cup\hat{\reward}_{\episode'}}{\hat{\Policy}^*_\episode}{\timestep}(\state_0)}_{\overset{(a)}{\leq} 4 \epsilon_\episode}
> 0
\end{align*}
where we applied (a) \Cref{lemma_policyci}, (b) \Cref{policyci_lemma_2}, and (c) the definition of $\policyci_\episode$ and the fact that $\Policy \notin \policyci_\episode$. Consequently, $\Policy$ is suboptimal for at least some reward function $\hat{\reward}_{\episode'} \in \RecoveredFeasibleSetIterPrime$.
\end{proof}

\begin{corollary}\label{cor:policyci_sound}
For $\epsilon_0 = \frac{\Horizon}{10}$, for every $\episode\geq0$ it holds that both $\Policy^*, \hat{\Policy}^*_{\episode+1} \in \policyci_\episode$.
\end{corollary}

\begin{proof}
We show the statement by induction over $\episode$. For $\episode = 0$, we have $10\epsilon_0 = \Horizon$ and therefore $\policyci_0$ contains all policies. Assume that for $\episode-1$ the statement holds, i.e., $\Policy^*, \hat{\Policy}^*_{\episode} \in \policyci_{\episode-1}$, and consider $\episode$. By \Cref{policyci_lemma_1}, $\Policy^* \in \policyci_\episode$. Note, that $\hat{\Policy}^*_{\episode+1} \in \policyci_{\episode-1}$. Hence, by \Cref{policyci_lemma_2}, it follows that $\hat{\Policy}^*_{\episode+1} \in \policyci_\episode$ because it would be suboptimal otherwise which is a contradiction.
\end{proof}

The last result we need, is quantifying the size of the policy confidence set.
\begin{lemma}\label{policyci_lemma_3}
Under the good event $\GoodEvent$, let $\tilde{\reward} \in \argmin_{\reward\in\ExactFeasibleSet} \max_{\state,\action} (\reward(\state,\action) - \hat{\reward}_\episode(\state,\action))$, where $\hat{\reward}_\episode = \irlalg(\RecoveredFeasibleSetIter)$. If $\Policy \in \policyci_\episode$, then $\max_\state (\Valfun{\EstMDP\cup\tilde{\reward}}{*}{\timestep}(\state) - \Valfun{\EstMDP\cup\tilde{\reward}}{\Policy}{\timestep}(\state)) \leq 12 \epsilon_\episode$.
\end{lemma}

\begin{proof}
\begin{align*}
&\Valfun{\EstMDP\cup\tilde{\reward}}{*}{\timestep}(\state) - \Valfun{\EstMDP\cup\tilde{\reward}}{\Policy}{\timestep}(\state)
= \underbrace{\Valfun{\EstMDP\cup\tilde{\reward}}{*}{\timestep}(\state) - \Valfun{\EstMDP\cup\hat{\reward}_\episode}{*}{\timestep}(\state)}_{\leq \epsilon_\episode}
+ \underbrace{\Valfun{\EstMDP\cup\hat{\reward}_\episode}{*}{\timestep}(\state) - \Valfun{\EstMDP\cup\hat{\reward}_\episode}{\Policy}{\timestep}(\state)}_{\leq 10 \epsilon_\episode}
+ \underbrace{\Valfun{\EstMDP\cup\hat{\reward}_\episode}{\Policy}{\timestep}(\state) - \Valfun{\EstMDP\cup\tilde{\reward}}{\Policy}{\timestep}(\state)}_{\leq \epsilon_\episode} \epsilon_\episode
\leq 14 \epsilon_\episode
\end{align*}
\end{proof}

Next, we define the error upper bound based on the policy confidence set.
\begin{definition}
Using $\policyci_\episode$, we define recursively:
\begin{align*}
&\hat{E}_\episode^\Horizon(\state, \action) = 0\\
&\hat{E}_\episode^\timestep(\state, \action) = \min\Bigl((\Horizon-\timestep)\Rmax, \rewardci_\episode^\timestep(\state, \action) + \sum_{\state'} \EstTransitionModel(\state' | \state, \action) \max_{\Policy\in\policyci_{\episode-1}} \Policy(\action'|\state') \hat{E}_\episode^{\timestep+1}(\state', \action') \Bigr)
\end{align*}
where $\EstTransitionModel$ is the estimated transition model of the environment. In contrast to \Cref{def:ucrl_error}, the maximization is over policies in $\policyci_\episode$ rather than all actions.
\end{definition}

This definition allows us to derive results that are analogous to the problem independent case.

\begin{lemma}\label{lemma:ucrl_error_estmdp_dependent}
Under the good event $\GoodEvent$, for all policies $\Policy \in \policyci_\episode$ and reward functions $\reward$ and all $\state,\action \in \StateSpace \times \ActionSpace$:
\[
| \Qfun{\EstMDP\cup\reward}{\Policy}{\timestep}(\state, \action) - \Qfun{\MDP\cup\reward}{\Policy}{\timestep}(\state, \action) | \leq \hat{E}_\episode^\timestep(\state, \action)
\]
\end{lemma}

\begin{proof}
The proof is the same as for \Cref{lemma:ucrl_error_estmdp}, restricting the set of policies to $\policyci_\episode$.
\end{proof}

\begin{lemma}\label{lemma:ucrl_error_estreward_dependent}
Under the good event $\GoodEvent$, for all reward function $\reward$, all policies $\Policy \in \policyci_\episode$, and all $\state,\action \in \StateSpace \times \ActionSpace$:
\[
| \Qfun{\EstMDP\cup\hat{\reward}}{\Policy}{\timestep}(\state, \action) - \Qfun{\EstMDP\cup\reward}{\Policy}{\timestep}(\state, \action) | \leq \hat{E}_\episode^\timestep(\state, \action)
\]
\end{lemma}

\begin{proof}
The proof is the same as for \Cref{lemma:ucrl_error_estreward}, restricting the set of policies to $\policyci_\episode$.
\end{proof}

\begin{lemma}\label{lemma:ucrl_error_real_transitions_dependent}
Under the good event $\GoodEvent$, we have for any $\state,\action,\timestep:$
\[
\hat{E}_\episode^{\timestep}(\state, \action)
\leq 2 \rewardci_\episode^\timestep(\state,\action) + \sum_{\state'} \TransitionModel(\state'|\state,\action) \max_{\Policy\in\policyci_{\episode-1}} \Policy(\action'|\state') \hat{E}_\episode^{\timestep+1}(\state', \action')
\]
\end{lemma}

\begin{proof}
The proof is the same as for \Cref{lemma:ucrl_error_real_transitions_dependent}.
\end{proof}

Finally, we can combine these results to analyze the algorithm's sample complexity.

\SampleComplexityProblemDependent*

\begin{proof}
First note that the analysis of \Cref{thm:simple_sample_complexity} still applies; so, in the worst case we get the same sample complexity. The key difference is that we no longer use the overall greedy policy w.r.t $E_\episode^\timestep$, but restrict ourselves to policies in $\policyci_\episode$.

Again, we consider the error
\[
e_\episode^{\Policy,\timestep}(\state, \action) \coloneqq | \Qfun{\MDP\cup\reward}{\Policy^*}{\timestep}(\state, \action) - \Qfun{\MDP\cup\reward}{\hat{\Policy}^*}{\timestep}(\state, \action) |
\]
where $\Policy^*$ is the true optimal policy in $\MDP\cup\reward$, and $\hat{\Policy}^*$ is the optimal policy in $\EstMDP\cup\hat{\reward}$, i.e., in the estimated MDP using the inferred reward function.

Similar, to the proof of \Cref{aceirl_correctness}, we can use \Cref{lemma:ucrl_error_estmdp_dependent} and \Cref{lemma:ucrl_error_estreward_dependent} to show for all policies $\Policy\in\policyci_\episode^\timestep$, that:
\[
e_\episode^{\Policy,\timestep}(\state, \action) \leq 4 \hat{E}_\episode^{\timestep}(\state, \action)
\]
which implies the correctness of the algorithm according to \Cref{cor:uniform_sampling_stopping} when stopping at
\begin{align}
\hat{E}_\episode^0(\state_0, \Policy_{\episode+1}(\state_0)) \leq \frac{\epsilon}{4}
\label{stopping_condition_1}
\end{align}
Now, consider the following condition for all $\state,\action,\timestep$:
\begin{align}
\rewardci^\timestep_\episode(\state, \action) \leq - A_{\MDP\cup\tilde{\reward}}^{*,\timestep}(\state, \action) \frac{\epsilon}{48 \epsilon_{\episode-1}}, \label{stopping_condition_2}
\end{align}
where $\tilde{\reward} \in \argmin_{\reward\in\ExactFeasibleSet} \max_{\timestep,\state,\action} (\reward_\timestep(\state,\action) - \hat{\reward}_{\episode,\timestep}(\state,\action))$.
We will (a) show that when this condition holds the previous stopping condition also holds, and (b) analyze after how many iterations this condition will certainly hold. Together this will yield the result.

To show that \Cref{stopping_condition_2} implies \Cref{stopping_condition_1}, we assume that \Cref{stopping_condition_2} holds. Then, we get by applying \Cref{lemma:ucrl_error_real_transitions_dependent} recursively:
\begin{align*}
\hat{E}_\episode^0(\state_0, \Policy_{\episode+1}(\state_0))
&\leq 2 \max_{\Policy \in \policyci_{\episode-1}} \max_{\action} \sum_{\timestep=0}^{\Horizon} \sum_{\state',\action'}  \Occupancy{\MDP}{\Policy}{0}{\timestep}(\state',\action';\state_0,\action) \rewardci^{\timestep}_\episode(\state',\action') \\
&\leq 2 \max_{\Policy \in \policyci_{\episode-1}} \max_{\action} \sum_{\timestep=0}^{\Horizon} \sum_{\state',\action'}  \Occupancy{\MDP}{\Policy}{0}{\timestep}(\state',\action';\state_0,\action) \left( - A_{\MDP\cup\tilde{\reward}}^{*,\timestep}(\state', \action') \frac{\epsilon}{48 \epsilon_{\episode-1}} \right) \\
&\overset{(a)}{\leq} 2 \max_{\Policy \in \policyci_{\episode-1}} (\Valfun{\MDP\cup\reward}{*}{0}(\state_0) - \Valfun{\MDP\cup\reward}{\Policy}{0}(\state_0) ) \frac{\epsilon}{48 \epsilon_{\episode-1}} \overset{(b)}{\leq} \frac{\epsilon}{4}
\end{align*}
where (a) uses \Cref{lemma:sum_of_losses} and (b) uses \Cref{policyci_lemma_3}.

Next, we analyze after how many iterations \Cref{stopping_condition_2} holds, which will give a lower bound on the sample complexity result. The argument proceeds similar to the proof of \Cref{thm:simple_sample_complexity}.

Before the algorithm terminates at $\tau$, we have for all $\episode < \tau$:
\begin{align*}
\min_{\state,\action,\timestep} (- A_{\EstMDP\cup\tilde{\reward}}^{*,\timestep}(\state, \action)) \frac{\epsilon}{48 \epsilon_{\episode-1}} < \max_{\state,\action,\timestep} \rewardci_\episode^\timestep(\state,\action) \leq \Horizon \Rmax \sqrt{\frac{2 \ell^\timestep_\episode(\state, \action)}{\max({N_\episode^\timestep}(\state,\action), )}}
\end{align*}
Using similar argument to the proof of \Cref{thm:simple_sample_complexity}, using the same pseudo-counts, we arrive at:
\begin{align*}
\min_{\state,\action,\timestep} (- A_{\MDP\cup\tilde{\reward}}^{*,\timestep}(\state, \action)) \frac{\epsilon}{48 \epsilon_{\tau-1}} \sqrt{\tau+1} \leq \Horizon \Rmax \sqrt{8 \StateSpaceSize\ActionSpaceSize \log(12 \StateSpaceSize \ActionSpaceSize \Horizon \tau^2 / \delta)}
\end{align*}
Again, we can use Lemma 15 by \cite{kaufmann2021adaptive} to find that
\[
\tau \leq \tilde{\mathcal{O}}\left(
\frac{\Horizon^3 \Rmax^2 \StateSpaceSize \ActionSpaceSize \epsilon_{\tau-1}^2}{\min_{\state,\action,\timestep} (A_{\MDP\cup\tilde{\reward}}^{*,\timestep}(\state, \action))^2 \epsilon^2} \right)
\]
\end{proof}

\subsection{Computing the Exploration Policy}
\label{app:optimization_problem}

To run \AlgNameShort, we need to solve the optimization problem:
\[
\pi_\episode^\timestep = \min_{\Policy} \max_{\hat{\Policy} \in \policyci_{\episode-1}} \sum_{\timestep=0}^{\Horizon} \sum_{\state',\action'} \Occupancy{\EstMDP}{\hat{\Policy}}{0}{\timestep}(\state',\action';\state_0) \hat{\rewardci}^{\timestep}_\episode(\state',\action' | \Policy)
\]

For simplicity let us denote the state visitation frequencies by
\begin{align*}
\mu_\timestep(\state,\action) &\coloneqq \Occupancy{\EstMDP}{\Policy}{0}{\timestep}(\state,\action;\state_0) \\
\hat{\mu}_\timestep(\state,\action) &\coloneqq \Occupancy{\EstMDP}{\hat{\Policy}}{0}{\timestep}(\state,\action;\state_0)
\end{align*}

Let us introduce the following matrix notation
\begin{align*}
\tilde{A} = \begin{bmatrix}
\identity & 0 & 0 & 0 & \dots & 0 \\
\EstTransitionModel & -\identity & 0 & 0 & \dots & 0 \\
0 & \EstTransitionModel & -\identity & 0 & \dots & 0 \\
& & \dots & & & \\
0 & 0 & \dots & 0 & \EstTransitionModel & -\identity \\
1 & 0 & 0 & \dots & 0 & 0 \\
0 & 1 & 0 & \dots & 0 & 0 \\
& & \dots & & & \\
0 & 0 & 0 & \dots & 1 & 0 \\
0 & 0 & 0 & \dots & 0 & 1
\end{bmatrix},
\hspace{1cm}
a = \begin{bmatrix}
\hat{\reward}_{\episode-1}^0 \\
\hat{\reward}_{\episode-1}^1 \\
\dots \\
\hat{\reward}_{\episode-1}^\Horizon
\end{bmatrix},
\hspace{1cm}
A = \begin{bmatrix}
A & 0 \\
a^T & -1
\end{bmatrix},
\end{align*}

\begin{align*}
x = \begin{bmatrix}
\mu_0 \\
\mu_1 \\
\dots \\
\mu_\Horizon \\
t
\end{bmatrix},
\hspace{1cm}
\hat{x} = \begin{bmatrix}
\hat{\mu}_0 \\
\hat{\mu}_1 \\
\dots \\
\hat{\mu}_\Horizon
\end{bmatrix},
\hspace{1cm}
b = \begin{bmatrix}
\bar{\mu}_0 \\
0 \\
\dots \\
0 \\
1 \\
\dots \\
1 \\
-10 \epsilon_{\episode-1}
\end{bmatrix},
\hspace{1cm}
c = \begin{bmatrix}
\rewardci_0 \\
\rewardci_1 \\
\dots \\
\rewardci_\Horizon \\
1
\end{bmatrix},
\hspace{1cm}
\end{align*}
where $\bar{\mu}_0$ is the actual initial state distribution of the environment (which we assume to know). We can now write the inner maximization problem above as a linear program:

\begin{align*}
\max_x ~~ c^T x
~~\text{ s.t. }~~
A x = b, ~~
x \geq 0
\end{align*}

The corresponding dual problem is:

\begin{align*}
\min_y ~~ b^T y
~~\text{ s.t. }~~
A^T y \geq c
\end{align*}

Using this we can write the full min-max problem as:

\begin{align*}
\min_{\hat{x},y} ~~ b^T y
~~\text{ s.t. }~~
A^T y \geq c(x), ~~
\tilde{A} x = b, ~~
x \geq 0
\end{align*}
which is a convex optimization problem, if we use:
\[
\rewardci_\timestep(\state, \action) = 2 (\Horizon-\timestep) \Rmax \sqrt{\frac{2 \log\left( 24 \StateSpaceSize \ActionSpaceSize \Horizon (\max(1, {n^\timestep_\episode}(\state, \action)))^2 / \delta \right)}{\max(1, {\hat{n}^\timestep_{\episode+1}}(\state,\action)}})
\]
where $\hat{n}^\timestep_{\episode+1}(\state,\action) = n^\timestep_{\episode}(\state,\action) + \mu_\timestep(\state,\action) * N_E$ is the number of times we expect $\timestep, \state, \action$ to be visited at the next iteration.

Solving this optimization problem yields the state-visitation frequencies $\hat{\mu}_\episode(\state,\action)$. We can then find the exploration policy that induces these state-visitations simply as:
\[
\Policy_{\episode,\timestep}(\action|\state) \coloneqq \frac{\hat{\mu}_\episode^\timestep(\state,\action)}{\sum_{\action'} \hat{\mu}_\episode^\timestep(\state,\action')}.
\]

\section{Experimental Details}
\label{app:experiment_details}

In this section, we provide more details on our experiments. We discuss the environments in detail (\Cref{app:environment_details}), provide some information on the implementation and the libraries and computational resources we used (\Cref{app:implementation_details}), and we provide more full plots of all experiments we discussed in the main paper (\Cref{app:additional_results}).

\subsection{Details on the Environments}
\label{app:environment_details}

\textbf{Four Paths.} The four paths environment has $41$ states and $4$ actions:
\[\StateSpace = \{ c, l_1, \dots, l_{10}, u_1, \dots, u_{10}, r_1, \dots, r_{10}, d_1, \dots, d_{10} \}, \qquad \ActionSpace = \{ \action_1, \action_2, \action_3, \action_4 \},
\]
and a time horizon of $\Horizon=20$. The agent starts in the center state $c$, from which can move in four directions: left ($\action_1$), up ($\action_2$), right ($\action_3$), or down ($\action_4$). Each action $\action_i$ has a probability $p_i$ of failing. If an action fails it moves in the opposite direction. $p_1, \dots, p_4$ are sampled uniformly from $(0, 0.3)$. One of the states $(l_{10}, u_{10}, r_{10}, d_{10})$ is chosen as the goal state at random. The reward in the goal state is $1$, all other rewards are $0$.

\textbf{Double Chain.}
The \emph{Double Chain} MDP, proposed by \citet{kaufmann2021adaptive}, consists of $L$ states $\StateSpace = \{\state_0, \dots, \state_{L-1}\}$, and two actions $\ActionSpace = \{\text{left}, \text{right}\}$,  which correspond to a transition to the left or to the right. When the agent takes an action, there is a $0.1$ probability of moving to the other direction. The state $\state_{L-1}$ has reward $1$, all other states have reward $0$, and the agent starts in the center of the chain at $\state_{(L - 1)/2}$. We choose $L=31$, similar to \citet{kaufmann2021adaptive}. The environment has horizon $\Horizon=20$.

\textbf{Chain.}
The \emph{Chain} MDP, proposed by \citet{metelli2021provably} has $6$ states
$\StateSpace=\{\state_1,\state_2,\state_3,\state_4,\state_5,\state_u\}$ and $10$ actions $\ActionSpace=\{\action_1,\dots,\action_{10}\}$. The agent starts in a random initial state. Taking action $\action_{10}$ moves it right along the chain with probability $0.7$ and to state $\state_u$ with probability $0.3$. Any other action moves the agent right with probability $0.3$ and to state $\state_u$ with probability $0.7$. If the agent is in state $\state_u$, action $\action_{10}$ moves it back to state $\state_1$ with probability $0.05$. Any other action moves it to $\state_1$ with probability $0.01$. The reward is $1$ in all states except $\state_u$ where the reward is $0$. \citet{metelli2021provably} provide an illustration of the environment in Figure 3. We choose $\Horizon=10$ for the chain.

\textbf{Gridworld.}
The \emph{Gridworld}, proposed by \citet{metelli2021provably}, is a $3\times3$ gridworld with an obstacle in the center cell $(2,2)$ and a goal cell at the right center cell $(2, 1)$. The agent starts in a random non-goal cell, and it has $4$ action one to move in each direction. If the agent takes an action with probability $0.3$ the action fails and the agent moves in a random direction instead. If the agent is in the center cell $(2,2)$ which has the obstacle, if the agent would move right it instead stays in the center cell with probability $0.8$. The reward in the goal cell is $1$, all other rewards are $0$. \citet{metelli2021provably} provide an illustration of the gridworld in Figure 6. We choose $\Horizon=10$ for the gridworld.

\textbf{Random MDPs.} We generate random MDPs by uniformly sampling an initial state distribution and transition matrix and normalizing them. The rewards are sampled uniformly between $0$ and $1$. Our random MDPs have $9$ states, $4$ actions and horizon $10$.

\subsection{Implementation Details}
\label{app:implementation_details}

We provide a full implementation of \AlgNameShort in Python, using multiple open sources libraries, including \texttt{cvxpy} and the SCS optimizer \citep{diamond2016cvxpy,scs} for solving the optimization problem in \Cref{app:optimization_problem}, and standard libraries for numerical computing, including \texttt{numpy}, and \texttt{scipy}. We choose Maximum Entropy IRL \citep{ziebart2008maximum} as an IRL algorithm, but \AlgNameShort is agnostic to this choice.

We ran experiments in parallel on a server with two 64 Core AMD EPYC 7742 2.25GHz processors. We estimate a total wall-clock time of less than 48 hours for running all experiments presented in this paper, including $50$ random seeds each.

\subsection{Additional Results}
\label{app:additional_results}

We provide full learning curves for all experiments discussed in the main paper in \Cref{fig:app_results}.

\begin{figure}
    \newlength{\appplotwidth}
    \setlength{\appplotwidth}{0.33\linewidth}
    \centering
    \includegraphics[width=\appplotwidth]{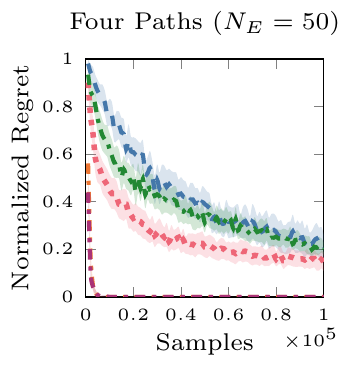}\hfill
    \includegraphics[width=\appplotwidth]{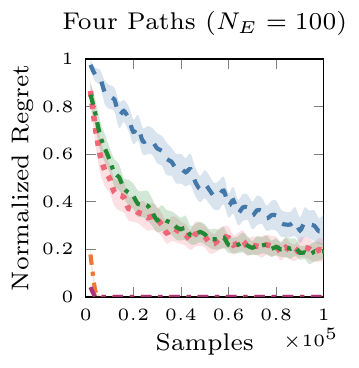}\hfill
    \includegraphics[width=\appplotwidth]{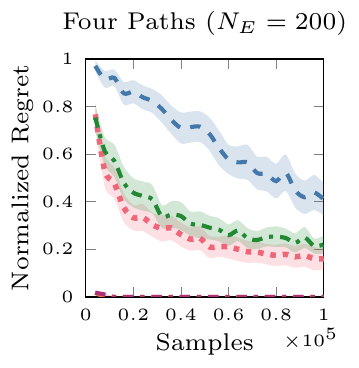} \\
    \includegraphics[width=\appplotwidth]{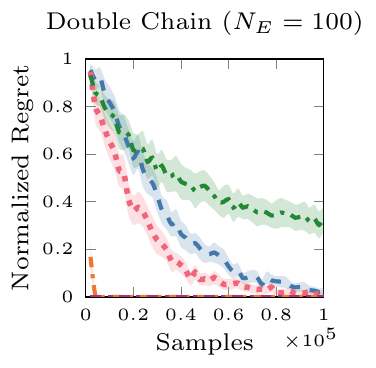}\hfill
    \includegraphics[width=\appplotwidth]{figures/appendix/double_chain_100ep.pdf}\hfill
    \includegraphics[width=\appplotwidth]{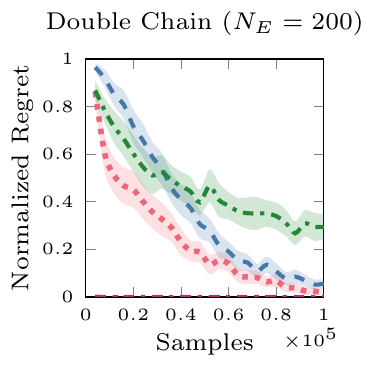} \\
    \includegraphics[width=\appplotwidth]{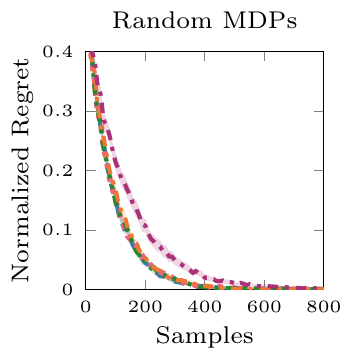}\hfill
    \includegraphics[width=\appplotwidth]{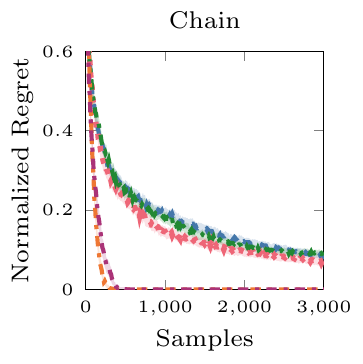}\hfill
    \includegraphics[width=\appplotwidth]{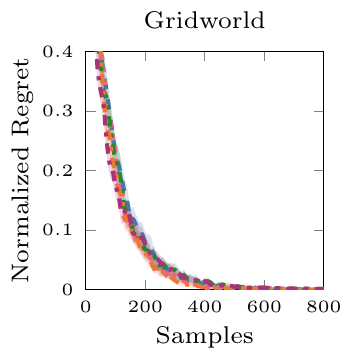} \\
    \includegraphics[width=0.9\linewidth]{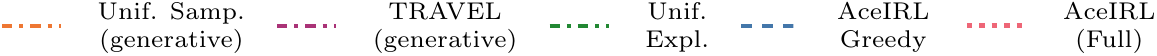}
    \caption{Full learning curves for all experiments shown in \Cref{tab:results}. Similar to \Cref{fig:results}, we show the mean and $95\%$ confidence intervals computed over $50$ random seeds. In addition to the exploration algorithms, we also show uniform sampling and TRAVEL which are much faster in most cases because they have access to a generative model.}
    \label{fig:app_results}
\end{figure}

\section{Connection to Reward-free Exploration}\label{app:reward-free-exploration}

In the \emph{reward-free exploration} problem, introduced by \citet{jin2020reward}, the agent explores an \MDPnoR to learn a transition model. In each iteration it chooses a new exploration policy based on previous data. The goal is to ensure that if the agent is given a reward function $\reward$ after the exploration phase it can find a good policy using its transition model. \citet{jin2020reward} formalize this goal as reducing the error:
\[
\Valfun{\MDP\cup\reward}{\Policy^*}{0}(\state_0) - \Valfun{\MDP\cup\reward}{\hat{\Policy}^*}{0},
\]
where $\hat{\Policy}^*$ is the optimal policy in the estimated MDP $\EstMDP\cup\reward$. Note the striking similarity between this problem, and the active IRL problem, we study in this paper. We want to reduce a similar error (cf. \Cref{def:correct}), but we have additional information about the reward in form of the expert policy.

The \emph{Reward-free UCRL} algorithm, proposed by \citet{kaufmann2021adaptive}, is essentially analogous to \AlgNameShort Greedy (\Cref{sec:stopping-condition}). Reward-free UCRL explores greedily with respect to an upper bound on the value function error. However, the exploration policy needs to be updated after each episode to adapt to the new uncertainty estimates. This might be expensive or not possible in practice. Instead, we could consider a \emph{batched} version of reward-free exploration, where in each iteration the agent explores for $N_E$ episodes, similar to our Active IRL problem. In this setting, a greedy policy w.r.t. uncertainty is suboptimal because it does not adapt to the reduced uncertainty over the $N_E$ episodes.

Instead, we can consider reducing the expected uncertainty at the next iteration, similar to our discussion in \Cref{sec:stopping-condition-dep}. If our error estimate is denoted by $E_\episode(\state,\action)$, we do no longer act greedily w.r.t. $E_\episode$. Instead we try to estimate the error at the next iteration $\hat{E}_{\episode+1}(\state,\action | \Policy)$ as a function of the policy and try to select the policy that reduces this error. In the tabular case, we can formulate this as a convex optimization problem, analogous to \Cref{app:optimization_problem}. We call this adaptation of \AlgNameShort to the reward-free exploration problem \emph{Ace-RF}.

\Cref{fig:rfe_results} shows illustrative results of this algorithm in the batched reward-free exploration setting in the \emph{Double Chain} environment. We find that for larger batch sizes, choosing an exploration policy that reduces future uncertainty is significantly better than reward-free UCRL.

\begin{figure}
    \newlength{\rfeplotwidth}
    \setlength{\rfeplotwidth}{0.2\linewidth}
    \centering
    \includegraphics[width=\rfeplotwidth]{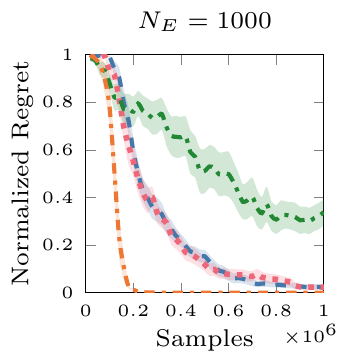}\hfill
    \includegraphics[width=\rfeplotwidth]{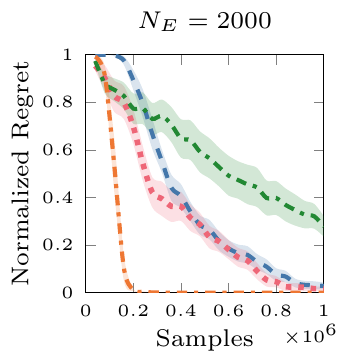}\hfill
    \includegraphics[width=\rfeplotwidth]{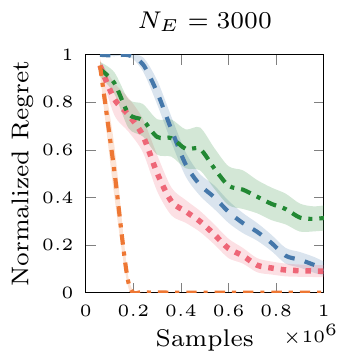}\hfill
    \includegraphics[width=\rfeplotwidth]{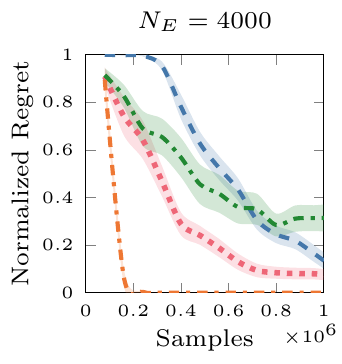}\hfill
    \includegraphics[width=\rfeplotwidth]{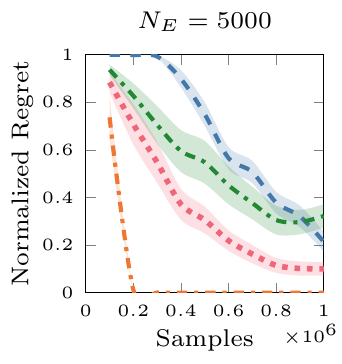} \\
    \includegraphics[width=0.7\linewidth]{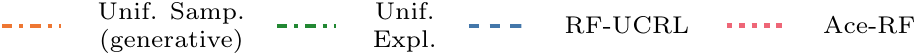}
    \caption{Illustrative experiments for reward-free exploration in the \emph{Double Chain} environment proposed by \citet{kaufmann2021adaptive}. The difference to our Active IRL setting is that the agent does not have access to the expert policy during exploration, but still tries to learn a good model of the environment. During testing it then gets access to the reward function, and the regret measures the suboptimality of the policy trained in the agent's transition model. We find that the ideas used in \AlgNameShort are also useful for batched reward-free exploration with larget $N_E$.}
    \label{fig:rfe_results}
\end{figure}

\end{document}